\documentclass{article}
\usepackage{microtype}
\usepackage{graphicx}
\usepackage{subfigure}
\usepackage{booktabs} 

\usepackage{wrapfig}
\usepackage{hyperref}
\usepackage{url}
\usepackage{amsmath}
\usepackage{algorithm}        
\usepackage{algorithmic}      
\usepackage{amsmath}
\usepackage{amssymb}
\usepackage{wrapfig}

\usepackage{amsthm}
\usepackage{graphicx}
\theoremstyle{plain}
\usepackage[table]{xcolor}  
\usepackage{booktabs}
\usepackage{ragged2e}      

\newtheorem{lemma}{Lemma}
\newtheorem{proposition}{Proposition}

\newtheorem{assumption}{Assumption}[section]
\usepackage{graphicx}
\usepackage{booktabs}
\usepackage{multirow}
\usepackage{enumitem}
\usepackage{xcolor}
\usepackage{comment}
\usepackage{tcolorbox}
\tcbuselibrary{theorems}

\usepackage{tabularx}
\usepackage{makecell} 
\usepackage{bm}
\usepackage{textcomp} 

\newtcbtheorem[number within=section]{graytheorem}{Theorem}{
  colback=gray!10,
  colframe=gray!50,
  fonttitle=\bfseries,
  boxrule=0.5pt,
  arc=2pt,
  left=6pt,
  right=6pt,
  top=6pt,
  bottom=6pt
}{thm}
\usepackage[preprint]{neurips_2026}

\usepackage[utf8]{inputenc} 
\usepackage[T1]{fontenc}    
\usepackage{hyperref}       
\usepackage{url}            
\usepackage{booktabs}       
\usepackage{amsfonts}       
\usepackage{nicefrac}       
\usepackage{microtype}      
\usepackage{xcolor}         
\usepackage{graphicx}
\usepackage{caption}
\title{SRG: \underline{S}core-based \underline{R}elaxation-guided \underline{G}eneration for Mixed Integer Linear Programming}

\author{%
  Ruobing Wang\textsuperscript{1} \quad
  Xin Li\textsuperscript{1,*} \quad
  Yujie Fang\textsuperscript{1} \quad
  Mingzhong Wang\textsuperscript{2} \\
  \textsuperscript{1}Beijing Institute of Technology, Beijing, China \\
  \textsuperscript{2}University of the Sunshine Coast, Australia \\
  \textsuperscript{*}Corresponding author: \texttt{xinli@bit.edu.cn}
}

\begin{document}

\maketitle

\begin{abstract}

We propose \underline{\textbf{S}}core-based \underline{\textbf{R}}elaxation-guided \underline{\textbf{G}}eneration (SRG), a generative framework based on an approximate formulation of relaxation-guided stochastic differential equations (SDEs) for mixed-integer linear programming. SRG employs a Transformer-based score network that incorporates feasibility and optimality signals into score modeling, encouraging the learned generative model to place more probability mass on feasible, high-quality regions of the solution space.  At inference time, SRG directly samples diverse candidate solutions from the learned score model without requiring any additional guidance module. These candidates are then used to construct compact trust-region subproblems for standard MILP solvers. Across multiple public benchmarks, SRG matches or improves upon the solution quality of the strongest learning-based baselines, with particularly strong gains in challenging candidate-generation settings. Moreover, SRG shows promising zero-shot transferability to unseen cross-scale and cross-problem instances, improving solver objectives and reducing search time in several cases through higher-quality initial candidates and compact trust-region search. 


\end{abstract}

\section{Introduction}
Mixed Integer Linear Programming (MILP) can formulate a wide range of combinatorial optimization problems, many of which are NP-hard~\citep{Karp1972}. MILP has been widely applied in production planning~\citep{ye2023deepaco}, resource allocation~\citep{ejaz2025comprehensivesurveylinearinteger}, and energy management~\citep{pinzon2017milp}, as well as in emerging applications across advanced manufacturing. However, due to the inherent NP-hardness of MILP, finding provably optimal solutions is computationally intractable for large-scale problems. Although exact algorithms such as branch-and-bound and cutting-plane methods \citep{1958Outline} provide optimality guarantees, they often incur prohibitive computational costs on large instances, limiting their practicality in time-sensitive real-world applications~\citep{bengio2020machinelearningcombinatorialoptimization}.

Recently, machine learning has shown great potential in accelerating MILP solving by leveraging structural similarities across problem instances and reusing decision knowledge~\citep{Gasse19,2020Hybrid,ling2024learningstopcutgeneration}. Among these approaches, Predict-and-Search (PaS) has emerged as a particularly effective paradigm~\citep{Han23,geng2025differentiable,liu2025apollomilpalternatingpredictioncorrectionneural}, which can significantly reduce the runtime of traditional MILP solvers. Specifically, PaS methods employ Graph Neural Networks (GNNs) over bipartite variable-constraint graphs to estimate the marginal probabilities of variable assignments, thereby identifying promising subspaces or critical structures where high-quality solutions are likely to reside. A conventional exact MILP solver (e.g., SCIP or Gurobi) is then invoked to search within these reduced regions to find near-optimal solutions efficiently.

Despite the success of existing PaS paradigms in solving general MILP problems, the models used in the prediction stage are typically trained by imitating solution labels, without explicitly incorporating feasibility and optimality information into the training process. As a result, the predictions may not be well aligned with feasible and high-quality regions of the MILP solution space. In addition, since PaS methods usually rely on deterministic or weakly stochastic predictions, the generated candidates may lack sufficient diversity, thereby restricting the coverage of promising solution regions. Consequently, the constructed trust regions can be imprecise or overly narrow~\citep{Han23}, increasing downstream search costs and ultimately limiting solution quality.
\begin{figure*}
\centering
\includegraphics[width=1\textwidth]{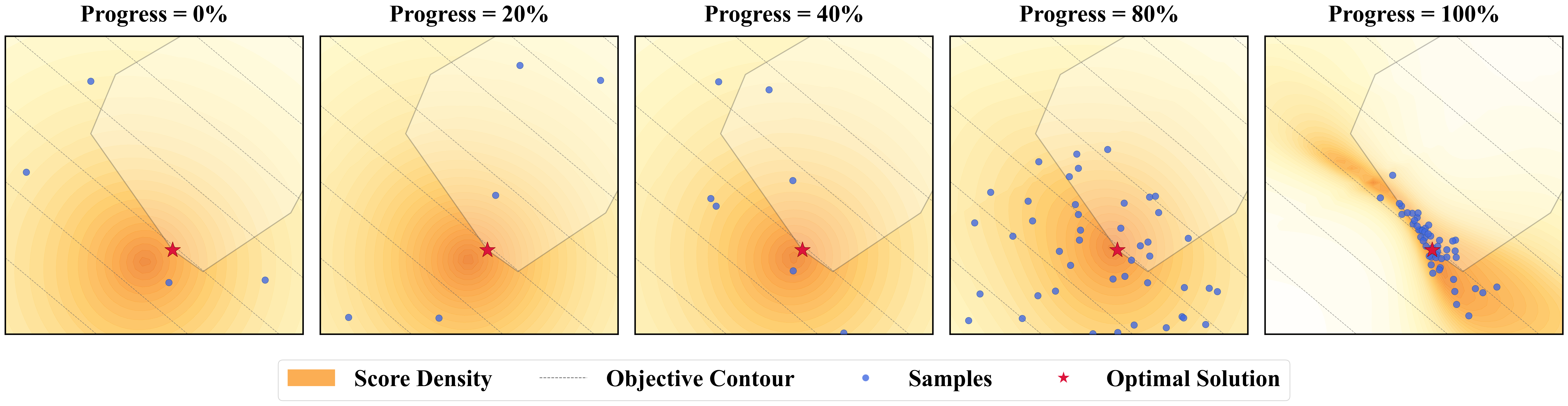}
\caption{Generation trajectory of our method on a 2D LP-relaxation toy experiment, which can be viewed as a continuous special case of MILP. 
\textbf{As the denoising process evolves, the grid-based score-density proxy, visualized by the yellow heatmap, highlights regions increasingly closer to the optimal solution (red star), while the generated samples (blue dots) move from an initially scattered state into a tight cluster near the optimum.} Additional experimental details are provided in Appx.~\ref{app:toyexperiments}.}
\label{fig:cross-problemtime}
\end{figure*} 

Motivated by these limitations, we propose \underline{\textbf{S}}core-based \underline{\textbf{R}}elaxation-guided \underline{\textbf{G}}eneration (SRG), a generative framework that incorporates feasibility and optimality information into score-based solution generation. 
Instead of relying solely on supervised imitation of solver labels, SRG formulates a regularized generative objective that combines the empirical solution distribution with relaxation-based feasibility and optimality signals. 
This objective can be interpreted as defining a refined target distribution, in which probability mass is reweighted toward feasible and high-quality regions. 
Based on this formulation, we derive an approximate relaxation-guided score and use it as a surrogate training target for a conditional score network.

The resulting model learns to generate solution candidates by denoising Gaussian noise toward regions favored by both the data distribution and the relaxation-guided signals, as illustrated in Figure~\ref{fig:cross-problemtime}. 
After training, SRG samples directly from the learned score model without requiring an additional guidance module at inference time. 
Repeated parallel sampling further provides diverse candidates, which are used to construct compact trust-region subproblems for standard MILP solvers. 
This design preserves the practical advantages of PaS-style solver refinement while improving candidate diversity and incorporating optimization-aware information during learning. 
\section{Related Work}

Mixed Integer Linear Programming (MILP) problems are challenging due to their hybrid decision spaces involving both continuous and discrete variables~\citep{bengio2020machinelearningcombinatorialoptimization}. Classical exact solvers, such as branch-and-bound~\citep{Land1960} and cutting-plane methods~\citep{repec:spr:sprchp:978-3-540-68279-0_4}, rely heavily on relaxation-based reasoning to reduce the search space and obtain valid bounds. In particular, relaxation techniques such as linear programming relaxation transform the original MILP into more tractable subproblems by relaxing integrality constraints~\citep{ichikawa2024controllingcontinuousrelaxationcombinatorial,cotter2018optimizationnondifferentiableconstraintsapplications}. The resulting relaxed problems provide useful feasibility and optimality information, which has long been central to the design of efficient MILP algorithms.

In recent years, machine learning-based methods have emerged as a promising direction for accelerating MILP solving. Among them, Predict-and-Search (PaS) frameworks~\citep{Han23,pmlr-v235-huang24f} have demonstrated strong empirical performance. PaS methods typically first predict high-quality candidate solutions or search regions, and then pass them to standard MILP solvers for subsequent search. By restricting the solver to promising regions around predicted solutions, PaS methods can significantly improve computational efficiency on large-scale MILPs.

Diffusion models can learn complex distributions and generate multiple diverse candidate solutions through repeated sampling \citep{rombach2022high,ho2020denoising}, thereby providing a flexible generative paradigm for MILP solving. Recently, diffusion-based solvers have achieved promising results on combinatorial optimization problems with special structures, such as the Traveling Salesman Problem (TSP) and Maximum Independent Set (MIS)~\citep{Sun23,li2023from}. However, directly extending existing diffusion-based solvers to general MILPs remains challenging, since most MILPs do not possess such special structures. Moreover, if feasibility information is not explicitly incorporated into the generative process, the sampled candidates may not be well aligned with feasible and high-quality solution regions, thereby limiting the final solution quality.

\section{Preliminaries}
\textbf{Mixed Integer Linear Programming (MILP)} optimizes a linear objective subject to linear constraints and integrality restrictions.
For an instance with \(m\) constraints and \(n\) variables, the standard MILP formulation is:
\begin{equation}
\min_{x \in \mathbb{R}^{n}} c^\top x \quad \text{s.t.} \quad A x \geq b, \; l \leq x \leq u, \; x_{j}\in\mathbb{Z}, \; \forall j\in I
\label{eq:milp}
\end{equation}
where \(A\in\mathbb{R}^{m\times n}\), \(c\in\mathbb{R}^{n}\), \(b\in\mathbb{R}^{m}\), \(l,u\in(\mathbb{R}\cup\{\pm \infty\})^{n}\), and
\(I\subseteq\{1,\dots,n\}\) denotes the set of indices for integer-constrained variables. The feasible region is:
\begin{equation*}
D \;=\;
\Bigl\{
x \in \mathbb{R}^{n}
\;\bigm|\;
A x \geq b,\;
l \le x \le u,\;
x_{j} \in \mathbb{Z},\; \forall j \in I
\Bigr\}.
\end{equation*}

\textbf{MILP Bipartite Graph Representation.}
Following~\citep{Gasse19}, any MILP instance $(A, b, c)$ can be represented as a bipartite graph $\mathcal{G} = (\mathcal{V}, \mathcal{C}, \mathcal{E})$, where $\mathcal{V}$ and $\mathcal{C}$ denote variable and constraint nodes, respectively. This graph is then encoded into an embedding $\mathbf{g} = \tau_{\phi}(A, b, c)$ via a 2-layer GNN encoder $\tau_{\phi}$, which captures variable--constraint interactions in a compact form. The embedding $\mathbf{g}$ is used as the structural conditioning input for downstream learning modules.


\textbf{Continuous Relaxation for MILP.}
A standard tool in both classical solvers and learning-based methods is 
\emph{continuous relaxation}, where integrality constraints are relaxed and the variables are optimized in a continuous domain. 
For binary variables, this corresponds to relaxing $\{0,1\}$ to $[0,1]$. 
This relaxation allows both the objective value and constraint violations to be evaluated on continuous candidates.

Following relaxation- and penalty-based formulations~\citep{ichikawa2024controllingcontinuousrelaxationcombinatorial}, one can define a penalty-relaxed loss:
\begin{equation}
\hat{\ell}(x;\lambda)
=
\hat{f}(x)
+
\sum_{i=1}^{m} \lambda_i \hat{v}_i(x),
\label{eq:relaxation_loss}
\end{equation}
where $\hat{f}(x)=c^\top x$ is the relaxed objective, and $\hat{v}_i(x)$ measures the violation of the $i$-th constraint. 
For an inequality constraint $a_i^\top x \ge b_i$, there exists:
\begin{equation}
\hat{v}_i(x)
=
\max\{b_i-a_i^\top x,0\},
\end{equation}
while constraints of the form $a_i^\top x \le b_i$ can be handled symmetrically. 
The coefficient $\lambda_i\ge0$ controls the trade-off between objective optimization and constraint satisfaction. 
If equality constraints $h_j(x)=0$ are present, they can be incorporated through penalties such as $\hat{v}^{=}_j(x)=h_j(x)^2$.

Notably, Eq.~\eqref{eq:relaxation_loss} provides two complementary signals for a relaxed candidate $x$: 
\emph{objective quality}, through $\hat{f}(x)$, and \emph{feasibility}, through the weighted violation term. In this paper, \textbf{S}RG uses these objective and feasibility terms as guidance signals for score-based generative modeling, enabling the model to learn a distribution over high-quality solution regions and generate diverse candidates for downstream solver refinement.

\section{Our Method}
We propose \underline{\textbf{S}}core-based \underline{\textbf{R}}elaxation-guided \underline{\textbf{G}}eneration (SRG), which formulates MILP solution prediction as sampling from a refined target distribution. The framework consists of two core components: (1) \underline{\textit{Target Formulation}}, which incorporates relaxation terms to construct a reshaped probability density that concentrates mass on feasible and high-quality solutions; and (2) \underline{\textit{Reverse Sampling}}, which samples from this target distribution via a reverse-time SDE, injecting gradients derived from approximate feasibility and optimality terms into the denoising dynamics to steer generation toward desirable solution regions.

\subsection{Approximate Relaxation-guided Score Function}

Let $p_{\mathrm{data}}^\sigma(x|\mathbf{g})$  denote a continuous reference density of high-quality solutions for MILP instance $\mathbf{g}$. Given a reference optimal solution $x^\star$, we define:
\[
p_{\mathrm{data}}^\sigma(x|\mathbf{g})=\mathcal N(x;x^\star,\sigma^2 I).
\]
$p_{\mathrm{data}}^\sigma(x|\mathbf{g})$ assigns high probability in the neighborhood of $x^*$. For notational simplicity, we write $p_{\text{data}}(x \mid \mathbf{g})$ for $p_{\mathrm{data}}^\sigma(x \mid \mathbf{g})$ throughout the paper.

Leveraging the bipartite graph embedding $\mathbf{g}$ as the structural condition, we formulate MILP solution prediction as a conditional generative task. Our goal is to learn a parameterized distribution $q_\theta(x|\mathbf{g})$ that approximates the target distribution of optimal solutions $p_{\mathrm{data}}^\sigma(x|\mathbf{g})$. Unlike deterministic approaches that output a single prediction, this probabilistic formulation enables sampling of diverse candidate solutions from the learned distribution, thereby increasing the likelihood of finding high-quality solutions. A standard approach to achieve this is to minimize the KL divergence:
\begin{equation}
\min_{q} D_{\mathrm{KL}}\left( q_{\theta}(x|\mathbf{g}) \|p_{\mathrm{data}}(x|\mathbf{g}) \right)
\label{eq:baseline_kl}
\end{equation}
However, directly minimizing Eq.~(\ref{eq:baseline_kl}) remains a purely data-driven approach that does not explicitly leverage feasibility or optimality signals intrinsic to the MILP structure, potentially causing generated solutions to deviate from feasible and near-optimal regions under the same finite sampling budget. Moreover, relying solely on data-driven alignment can make the model overly dependent on dataset-specific correlations, specifically given finite training samples and limited solver computation time \citep{Sun23}, increasing susceptibility to dataset bias.

Motivated by Eq.~\ref{eq:relaxation_loss}, which establishes that relaxed solutions possess desirable feasibility and optimality guarantees for MILP problems, we propose to integrate relaxation-based objective and feasibility terms as explicit guidance into probability distribution modeling.

Specifically, we introduce an optimality reference term $\mathcal{O}(x) = \|c \odot (x - x^*)\|_1$, where $\odot$ denotes element-wise product, and a feasibility reference term $\mathcal{P}(x)=\lambda\|\max\{b-Ax,0\}\|_1$, where $x^*$ is the optimum. Unlike the classical Lagrangian penalty $\max_{\lambda\ge 0}\lambda^{\top}(b-Ax)$, which becomes unbounded on infeasible points (as the term $\longrightarrow \infty$ when $Ax \not\geq b$), our $\mathcal{P}(x)$ provides a finite-valued signal for all $x \in \mathbb{R}^n$. Incorporating these terms yields the following optimization objective:
\begin{equation}
\min_{q_\theta} \left[ D_{\mathrm{KL}}(q_\theta(x|\mathbf{g}) \| p_{\text{data}}(x|\mathbf{g})) + \mathbb{E}_{x \sim q_\theta}\left[\gamma_o \mathcal{O}(x) + \gamma_c \mathcal{P}(x)\right]\right] 
\label{eq:total_optimization}
\end{equation}
where $\gamma_o, \gamma_c \in \mathbb{R}^{+}$ are guidance coefficients. The following proposition establishes that optimizing Eq.~(\ref{eq:total_optimization}) is theoretically equivalent to performing KL alignment toward a refined target distribution $\tilde{p}(x \mid \mathbf{g})$, concentrated on high-quality regions~\citep{ouyang2022traininglanguagemodelsfollow,rafailov2024directpreferenceoptimizationlanguage}.
\begin{proposition}[\underline{\textbf{O}}ptimization \underline{\textbf{E}}quivalence]
\label{prop:equivalence_MAIN}
Minimizing the regularized objective in Eq.~(\ref{eq:total_optimization}) is equivalent to minimizing
\begin{equation}
\min_{q_\theta} D_{\mathrm{KL}}\left(q_\theta(x|\mathbf{g}) \| \tilde{p}(x|\mathbf{g})\right)
\end{equation}
where the refined target distribution is
\begin{equation}
\tilde{p}(x|\mathbf{g}) := \frac{1}{Z} p_{\text{data}}(x|\mathbf{g}) \exp\left(-\gamma_o \mathcal{O}(x) - \gamma_c \mathcal{P}(x)\right),\, x\sim \tilde{p}
\label{eq:optimization_eq}
\end{equation}
and $Z$ is the normalization constant.
\end{proposition}
\textit{Proof.} See Appx.~\ref{app:lsde}.

\textbf{Reverse-time SDE for Gaussian distribution to $\tilde{p}$.}
To sample from the target distribution $\tilde{p}$ defined in Proposition~\ref{prop:equivalence_MAIN}, we adopt a score-based generative modeling framework~\citep{song2021scorebased}, starting from a Gaussian distribution. The following proposition provides the corresponding reverse-time SDE and its score function decomposition; The detailed derivation is in Appendix~\ref{app:proof_reverse_sde}.
\begin{proposition}[\underline{\textbf{R}}everse-time \underline{\textbf{S}}DE and its \underline{\textbf{S}}core \underline{\textbf{F}}unction  $\nabla_{x_t} \log \tilde p_t(x_t |\mathbf{g})$]\label{prop:reverse_sde}
Assume $\tilde p_t(x \mid g) > 0$ for all $(x,t) \in \mathbb{R}^n \times [0,T]$ and that $\nabla_x \log \tilde p_t(x \mid g)$ is well-defined. Then the reverse-time process
\begin{equation}
    d x_t = \Big[ -\tfrac{1}{2}\beta(t)\, x_t - \beta(t)\, \nabla_{x_t}\! \log \tilde p_t(x_t \mid \mathbf{g}) \Big]\, dt + \sqrt{\beta(t)}\, d\bar W_t
\label{eq:reverse_sde1}
\end{equation}
has marginals equal to those of the forward process. Here $d\bar W_t$ denotes a standard Wiener process running backward in time from $T$ to $0$. For every $t \in [0,T]$, the score of $\tilde p_t$ in Eq.~\ref{eq:reverse_sde1} admits the following decomposition:
\begin{equation}
    \nabla_{x_t} \log \tilde p_t(x_t |\mathbf{g})
    = 
    \nabla_{x_t} \log p_t(x_t|\mathbf{g})
    + 
    \nabla_{x_t} \log \mathbb{E}_{x_0 \sim p(x_0 \mid x_t, \mathbf{g})}\! \big[ w(x_0) \big],
\label{eq:score_exact}
\end{equation}
where $p_t(\cdot \mid g)$ is the marginal at time $t$ of the forward process started at $x_0 \sim p_{\text{data}}(\cdot \mid g)$, and $p(x_0 \mid x_t, g)$ is the corresponding posterior under that process.
\label{prop:exactprop}
\end{proposition}

\textbf{Surrogate Approximation for the Score Function.} 
The guided score in Eq.~\ref{eq:score_exact} requires $\nabla_{x_t} \log \mathbb{E}_{x_0 \sim p(x_0 \mid x_t, \mathbf{g})}\! \big[ w(x_0) \big]$, which is intractable because the posterior $p(x_0|x_t,g)$ has no closed-form expression. Following the classifier-guidance line of work~\citep{dhariwal2021diffusion,chung2024diffusionposteriorsamplinggeneral}, we evaluate the guidance gradient at the current noisy sample $x_t$, motivated by two observations: \underline{(i)} Tweedie's formula~\citep{Efron2011TweediesFA} implies $\mu(x_t):=\mathbb{E}[x_0\mid x_t]\approx x_t$ in the low-noise regime ($\sigma_t\to 0$), and \underline{(ii)} although $w(x_0)=\exp(-\gamma_o\mathcal{O}(x_0)-\gamma_c\mathcal{P}(x_0))$ is non-smooth (due to the L1 kinks in $\mathcal{O}$ and ReLU boundaries in $\mathcal{P}$), Gaussian convolution in the forward process renders the conditional expectation $\mathbb{E}[w(x_0)\mid x_t]$ smooth in $x_t$, so the gradient is well-defined. This yields the surrogate
\begin{equation}
\nabla_{x_t} \log \mathbb{E}_{x_0 \sim p(x_0 \mid x_t, \mathbf{g})}\! \big[ w(x_0) \big]
\;\approx\; 
\nabla_{x_t} \log w(x_t).
\label{eq:taylor_score_approx}
\end{equation}
Substituting $w(x)=\exp(-\gamma_o\mathcal{O}(x)-\gamma_c\mathcal{P}(x))$ into Eq.~\ref{eq:taylor_score_approx} (with $\nabla\mathcal{O}, \nabla\mathcal{P}$ taken as Clarke subgradients at non-smooth points), we obtain the surrogate guided score:
\begin{equation} 
s^{*}(x_t,t,g) \approx \nabla_{x_t}\log p_t(x_t|g) -\gamma_o\nabla_{x_t}\mathcal{O}(x_t) -\gamma_c\nabla_{x_t}\mathcal{P}(x_t). 
\label{eq:modified_score_function_MAIN} 
\end{equation} 
We emphasize that the derivation above is intended as an intuitive motivation, and Eq.~\ref{eq:modified_score_function_MAIN} is a tractable surrogate that captures the dominant feasibility and optimality signals of the refined score, but does not exactly characterize it at all noise levels. Instead, we treat Eq.~\ref{eq:modified_score_function_MAIN} as a practical training target whose effectiveness is empirically validated in our main solving results.
\subsection{Conditional Score Network in $\epsilon$-space}
\paragraph{Transformer-based $s_{\epsilon}(x_t, t, \mathbf{g})$.} To learn the guided score $s^*$ in Eq.~(\ref{eq:modified_score_function_MAIN}), motivated by the favorable scalability of the Diffusion Transformer (DiT) architecture~\citep{peebles2023scalablediffusionmodelstransformers}, we instantiate the conditional score network in $\epsilon$-space $s_{\epsilon}(x_t, t, \mathbf{g})$ as a lightweight transformer. 

Each MILP instance is first encoded into a structural embedding $\mathbf{g} = \tau_{\phi}(A, b, c) \in \mathbb{R}^{d_{n}}$ via a GNN encoder. The noisy solution $x_t \in \mathbb{R}^{B \times C \times H \times W}$ is partitioned into non-overlapping $p\times p$ patches and projected into a token sequence $\mathbf{z} \in \mathbb{R}^{B \times N \times D}$ with $N = \lfloor H/p \rfloor \cdot \lfloor W/p \rfloor$, augmented by learnable positional embeddings. The timestep $t$ and the structural embedding $\mathbf{g}$ are each projected to $\mathbb{R}^{D}$ and summed into a conditioning vector $\mathbf{c} \in \mathbb{R}^{B \times D}$.Here, $B$ is the batch size, $C$ is the number of channels, $(H, W)$ are the spatial dimensions of the noisy solution, $p$ is the patch size, $N = \lfloor H/p\rfloor \cdot \lfloor W/p\rfloor$ is the number of patch tokens, and $D$ is the token embedding dimension.

Our network stacks $L$ transformer blocks, each applying Adaptive Layer Normalization (AdaLN)-conditioned self-attention and an MLP with zero-initialized gating (AdaLN-Zero):
\begin{equation}
\mathbf{z}' = \mathbf{z} + \alpha_1 \odot \mathrm{MHA}\!\left(\mathrm{AdaLN}(\mathbf{z}, \mathbf{c})\right), \qquad
\mathbf{z}'' = \mathbf{z}' + \alpha_2 \odot \mathrm{MLP}\!\left(\mathrm{AdaLN}(\mathbf{z}', \mathbf{c})\right),
\end{equation}
where $\mathrm{AdaLN}(\mathbf{z}, \mathbf{c}) = (1 + \boldsymbol{\gamma}(\mathbf{c})) \odot \mathrm{LayerNorm}(\mathbf{z}) + \boldsymbol{\beta}(\mathbf{c})$ and the gates $\alpha_1, \alpha_2$ are zero-initialized for training stability. Self-attention lets every patch token attend to all others, capturing global inter-variable dependencies without the locality bias of convolutional architectures. After $L$ blocks, the output tokens are projected back to pixel space and rearranged into $\hat{\boldsymbol{\epsilon}} \in \mathbb{R}^{B \times C \times H \times W}$.

\paragraph{Cross-scale adaptation.} Although the network is trained at a fixed resolution and a fixed structural-embedding dimension $d_n$, we design mechanisms to enable inference on MILPs of arbitrary size without retraining: \underline{(i)} The learnable positional embedding is bilinearly interpolated to the test-time patch grid; \underline{(ii)} The unpatched output is resampled back to the original size of $x_t$, removing the requirement that $(H,W)$ be divisible by the patch size; \underline{(iii)} The embedding $\mathbf{g}' \in \mathbb{R}^{d_n'}$ is adapted to $d_n$ before the conditioning projection via 1D linear interpolation when $d_n' > d_n$ and zero-padding when $d_n' < d_n$. Combined with the natural length-flexibility of self-attention over token sequences, a single trained score network generalizes to MILP instances with varying scales.

\subsection{Training and Sampling}

\textbf{Training.} To learn the approximate relaxation-guided score function, we minimize the following score matching loss:
\begin{equation}
\mathcal{L}_{\text{Relaxed}} = \mathbb{E}\| s - s_{\text{target}} \|_2^2
\end{equation}
where $s_{\text{target}}$ is $s^*$ in Eq.~(\ref{eq:modified_score_function_MAIN}). 
Following~\citep{song2021scorebased}, we generate noisy samples via the forward process $x_t = \sqrt{\bar{\alpha}_t}\, x^* + \sqrt{1-\bar{\alpha}_t}\, \epsilon$ with $\epsilon \sim \mathcal{N}(0, I)$ \citep{ho2020denoising}. The objective deviation subgradient is $\nabla_{x_t} \mathcal{O}(x_t) = c \odot \operatorname{sign}(c \odot (x_t - x^*)) = c \odot \delta_{\mathcal{O}}$, and the constraint penalty gradient $\nabla_x \mathcal{P}(x) = -\lambda A^\top \mathbf{1}\!\left[ Ax < b \right]$, where $\mathbf{1}[\cdot]$ is the elementwise indicator.

Considering the discrete structure of the MILP solution space and the non-smoothness of $\nabla_{x_t} \mathcal{O}(x_t)$ and $\nabla_x \mathcal{P}(x)$, we perform the following simplifications to stabilize training. We approximate $\nabla_{x_t} \mathcal{O}(x_t)$ with $\nabla_{x_t} \mathcal{O}(\tilde{x}_t) = c \odot \tilde{\delta}_{\mathcal{O}} = c\odot \operatorname{sign}(c \odot (\tilde{x}_t - x^*))$, where the proxy $\tilde{x}_t$ is obtained by projecting integer coordinates in $x_{t_j}$ ($j \in I$) to their nearest discrete values while keeping continuous coordinates unchanged. We use these proxies to better reflect the discrete and hybrid structure of the MILP solution space. Meanwhile, we simplify the constraint penalty gradient $\nabla_{x_t} \mathcal{P}(x_t) = -\lambda A^\top 1_{Ax<b}$ to $\nabla_{x_t} \mathcal{\tilde{P}}(x_t) = -\lambda A^\top$, pushing all $x_t$ toward the feasible region. 

Finally, by substituting these terms, converting to noise-prediction parameterization via $\nabla_{x_t} \log p_t(x_t|\mathbf{g}) \approx -\epsilon_t / \sqrt{1 - \bar{\alpha}_t}$~\citep{song2021scorebased}, and multiplying through by $-\sqrt{1 - \bar{\alpha}_t}$, we obtain the simplified surrogate training objective $\mathcal{L}_\text{simplified}$ with the simplified stop gradient ($\texttt{sg}$) target in Eq.~\ref{eq:simplified_loss_time}, which is defined in $\epsilon$-space with Denoising Diffusion Probabilistic Model (DDPM) style learning~\citep{ho2020denoising} with the timestep variable $t$:
\begin{equation}
\begin{aligned}
\quad \quad \mathcal{L}_{\text{simplified}} = &\left\| s_\epsilon(x_t, t, \mathbf{g}) - \texttt{sg}\left( \epsilon_t + \textbf{v}_o+ \textbf{v}_c \right) \right\|_2^2 \\
\text{where} \quad \textbf{v}_o = \nabla_{x_t} \mathcal{O}(\tilde{x}_t) &= \gamma_o \sqrt{1 - \bar{\alpha}_t} c \odot \tilde{\delta}_{\mathcal{O}}, \quad \quad \textbf{v}_c(t) = \nabla_{x_t} \mathcal{\tilde{P}}(x_t) = -\gamma_c\sqrt{1 - \bar{\alpha}_t} \lambda A^\top
\label{eq:simplified_loss_time}
\end{aligned}
\end{equation}

\textbf{Adaptation rules for $\gamma_o$ and $\gamma_c$.}
In $L_{\text{simplified}}$ training, the terms $\textbf{v}_o$ and $\textbf{v}_c$ scale with the magnitudes of the problem data and timestep $t$, roughly with $\|c\|_2$ and $\|A\|_F$. Therefore, fixed $\gamma_o,\gamma_c$ may cause two main issues: \underline{\textbf{(1)}} the resulting guidance terms may dominate the base diffusion signal $\epsilon_t$, weakening the learning of the underlying data distribution; and \underline{\textbf{(2)}} optimization may become unstable. To mitigate these issues, we calibrate the effective guidance strength in an instance-adaptive and time-aware manner. Let
\[
\mathbf u_o \triangleq c\odot \tilde{\delta}_{\mathcal O},
\qquad
\mathbf u_c \triangleq -\lambda A^\top \mathbf 1,
\]
and define
\begin{equation}
\gamma_o
\leftarrow
\gamma_o\,\frac{\rho_o\sqrt{n}}{\|\mathbf u_o\|_2+\varepsilon} > 0,
\qquad
\gamma_c
\leftarrow
\gamma_c\,\frac{\rho_c\sqrt{n}}{\|\mathbf u_c\|_2+\varepsilon} > 0,
\label{eq:adaptive_gamma_main}
\end{equation}
where $\gamma_o,\gamma_c\in\mathbb R_{>0}$ are the guidance weights, $\rho_o,\rho_c\in\mathbb R_{>0}$ are guidance-strength parameters, and $\varepsilon>0$ is a small constant for numerical stability.

\textbf{Sampling.} At inference time, we sample \textit{solely} using the learned guided score model $s_{\theta}(\cdot)$, \textbf{without requiring any additional guidance computation, constraint evaluation, or external optimization components}. Specifically, we initialize from a prior $x_T \sim \mathcal{N}(0, I)$~\citep{ho2020denoising} and iteratively denoise using the trained score function to reach the guided distribution $\tilde{p}(x|\mathbf{g})$, which captures both feasibility and near-optimality. We then sample $\tilde{x}_0$, which remains relaxed, from $\tilde{p}(x|\mathbf{g})$ to construct the trust region for search. We use DDPM \citep{ho2020denoising} and DDIM \citep{song2022denoisingdiffusionimplicitmodels} samplers.

\textbf{Diverse Generation.} To leverage sampling diversity, we generate $k$ candidate solutions $\{ x^{j},\, j=1,\ldots, k \}$ in parallel using different random seeds. We then select $\tilde{x}^{*}$ as the candidate with the highest binary confidence. This strategy enhances robustness by exploring multiple high-probability regions before invoking downstream $L_1$ trust region search, following prior work \citep{Han23}.

\begin{table*}[htbp]
\centering
\small
\caption{Results on medium-scale MILP benchmarks. Results are averaged over \textit{50} test instances. ``$/$'' indicates subproblem $M_{\text{search}}(\tilde{x})$ is infeasible. ``$(\text{S})$'' and ``$(\text{G})$'' mean the underlying solver is SCIP and Gurobi, respectively. All baselines have a \underline{\textbf{1}}00s time limit to search for the best primal bound. }
\vspace{3pt}
\setlength{\tabcolsep}{4pt}
\scalebox{0.73}{
\begin{tabular}{l|cc|cc|cc|cc}
\toprule
\multirow{2}{*}{Method}  & \multicolumn{2}{c|}{\textbf{SC}} & \multicolumn{2}{c|}{\textbf{MIS}} & \multicolumn{2}{c|}{\textbf{CA}}  & \multicolumn{2}{c}{\textbf{CFL}} \\ 
\cmidrule(r){2-3} \cmidrule{4-5}  \cmidrule{6-7} \cmidrule{8-9} 
& Obj($\downarrow$) & $\mathrm{Gap}_{\mathrm{ref}}$($\downarrow$) & Obj($\uparrow$) & $\mathrm{Gap}_{\mathrm{ref}}$($\downarrow$) & Obj($\uparrow$) & $\mathrm{Gap}_{\mathrm{ref}}$($\downarrow$) & Obj($\downarrow$) & $\mathrm{Gap}_{\mathrm{ref}}$($\downarrow$) \\ 
\midrule
SCIP  \citep{bestuzheva2021scip} & $55.68 \pm 10.49$ & $0.00$  & $455.84 \pm 6.03$ & $0.00$ & $60797.78 \pm 1245.71$ & $0.00$ & $8858.28 \pm 446.19$ & $0.00$ \\ \hline
PaS(S) \citep{Han23}           & $55.68 \pm 10.49$ & $0.00$ & $455.32 \pm 6.80$ & $0.52$ & $60693.89 \pm 1184.15$ & $103.89$ & $8861.92 \pm 444.80$ & $3.64$ \\
ConPaS(S) \citep{pmlr-v235-huang24f}     & $55.68 \pm 10.49$ & $0.00$ & $455.44 \pm 6.38$ & $0.40$ & $\textbf{60693.89} \pm 1184.15$ & $\textbf{103.89}$ & $8861.93 \pm 444.80$ & $3.65$ \\
L2O-DiffILO(S) \citep{geng2025differentiable} & $55.68 \pm 10.49$ & $0.00$ & $452.02 \pm 9.26$ & $3.82$ & $58433.46 \pm 1468.10$ & $2364.32$ & $8858.28 \pm 446.19$ & $0.00$ \\
Apollo-MILP(S) \citep{liu2025apollomilpalternatingpredictioncorrectionneural}& $55.68 \pm 10.49$ & $0.00$ & $451.60 \pm 9.00$ & $4.24$ & $59131.99 \pm 1301.43$ & $1665.79$ & $8865.97 \pm 446.48$ & $7.69$ \\
\rowcolor{gray!20} \textbf{SRG(S)} \textbf{(Ours)} & $\textbf{55.68} \pm 10.49$ & $\textbf{0.00}$ & $\textbf{455.74} \pm 6.19$ & $\textbf{0.10}$ & $60349.62 \pm 1466.95$ & $448.16$ & $\textbf{8858.28} \pm 446.19$ & $\textbf{0.00}$ \\ \hline
Gurobi \citep{gurobi2023gurobi} & $55.68 \pm 10.49$ & $0.00$ & $455.90 \pm 5.93$ & $0.00$ & $62636.75 \pm 1239.87$ & $0.00$ & $8858.28 \pm 446.19$ & $0.00$ \\ \hline 
PaS(G)   \citep{Han23}         & $55.68 \pm 10.49$ & $0.00$ & $455.68 \pm 5.69$ & $0.22$ & $62919.95 \pm 1032.13$ & $-283.20$ & $8873.18 \pm 452.01$ & $14.90$ \\
ConPaS(G)  \citep{pmlr-v235-huang24f}    & $55.68 \pm 10.49$ & $0.00$ & $455.74 \pm 5.78$ & $0.16$ & $62990.09 \pm 1114.87$ & $-353.34$ & $8873.18 \pm 452.01$ & $14.90$ \\
L2O-DiffILO(G) \citep{geng2025differentiable} & $55.68 \pm 10.49$ & $0.00$ & $455.90 \pm 5.93$ & $0.00$ & $61847.63 \pm 1667.82$ & $789.12$ & $8858.28 \pm 446.19$ & $0.00$ \\
Apollo-MILP(G)  \citep{liu2025apollomilpalternatingpredictioncorrectionneural}& $55.68 \pm 10.49$ & $0.00$ & $455.90 \pm 5.93$ & $0.00$ & $61883.87 \pm 1278.21$ & $752.88$ & $8858.28 \pm 446.19$ & $0.00$ \\
\rowcolor{gray!20} \textbf{SRG(G)} \textbf{(Ours)} & $\textbf{55.68} \pm 10.49$ & $\textbf{0.00}$ & $\textbf{455.90} \pm 5.93$ & $\textbf{0.00}$ & $\textbf{63054.66} \pm 1183.44$ & $\textbf{-417.91}$ & $\textbf{8858.28} \pm 446.19$ & $\textbf{0.00}$ \\ \hline
\end{tabular}}
\label{tab:medium-scale}
\end{table*}

\begin{table*}[htbp]
\centering
\small
\caption{\underline{\textbf{C}}ross-\underline{\textbf{S}}cale results on large-scale MILP benchmarks. Results are averaged over \textit{20} test instances. All baseline solvers have a \underline{\textbf{6}}00s time limit to obtain the best primal bound. We \underline{\textbf{D}}irectly use the \underline{\textbf{P}}retrained SRG from \underline{\textbf{M}}edium-scale MILP benchmarks for large-scale solving.}
\vspace{3pt}
\setlength{\tabcolsep}{4pt}
\scalebox{0.73}{
\begin{tabular}{l|cc|cc|cc|cc}
\toprule
\multirow{2}{*}{Method}  & \multicolumn{2}{c|}{\textbf{SC}} & \multicolumn{2}{c|}{\textbf{MIS}} & \multicolumn{2}{c|}{\textbf{CA}}  & \multicolumn{2}{c}{\textbf{CFL}} \\ 
\cmidrule(r){2-3} \cmidrule{4-5}  \cmidrule{6-7} \cmidrule{8-9} 
& Obj($\downarrow$) & $\mathrm{Gap}_{\mathrm{ref}}$($\downarrow$) & Obj($\uparrow$) & $\mathrm{Gap}_{\mathrm{ref}}$($\downarrow$) & Obj($\uparrow$) & $\mathrm{Gap}_{\mathrm{ref}}$($\downarrow$) & Obj($\downarrow$) & $\mathrm{Gap}_{\mathrm{ref}}$($\downarrow$) \\ 
\midrule
SCIP \citep{bestuzheva2021scip}   & $13.85 \pm 3.89$ & $0.00$ & $912.40 \pm 7.00$ & $0.00$ & $11759.46 \pm 196.36$ & $0.00$ & $9667.92 \pm 503.55$ & $0.00$ \\ \hline
PaS(S) \citep{Han23}        & $13.90 \pm 3.94$ & $0.05$ & $911.15 \pm 8.31$ & $1.25$ & $11574.91 \pm 209.90$ & $184.55$ & $9672.72 \pm 501.40$ & $4.80$ \\  
ConPaS(S) \citep{pmlr-v235-huang24f}     & $13.90 \pm 3.94$ & $0.05$ & $910.90 \pm 8.56$ & $1.50$ & $11538.57 \pm 228.22$ & $220.89$ & $9672.72 \pm 501.40$ & $4.80$ \\
L2O-DiffILO(S) \citep{geng2025differentiable} & $13.90 \pm 3.88$ & $0.05$ & $910.55 \pm 8.04$ & $1.85$ & $11263.01 \pm 314.25$ & $496.45$ & $9667.92 \pm 503.55$ & $0.00$ \\  
Apollo-MILP(S) \citep{liu2025apollomilpalternatingpredictioncorrectionneural} & $13.85 \pm 3.89$ & $0.00$ & $910.35 \pm 8.14$ & $2.05$ & $11731.41 \pm 200.03$ & $229.14$ & $9667.92 \pm 503.55$ & $0.00$ \\
\rowcolor{gray!20} \textbf{SRG(S)} \textbf{(Ours)} & $\textbf{13.85} \pm 3.89$ & $\textbf{0.00}$ & $\textbf{911.60} \pm 8.21$ & $\textbf{0.80}$ & $\textbf{11743.23} \pm 143.75$ & $\textbf{16.23}$ & $\textbf{9667.92} \pm 503.55$ & $\textbf{0.00}$ \\ \hline
Gurobi \citep{gurobi2023gurobi} & $13.85 \pm 3.89$ & $0.00$ & $914.60 \pm 6.53$ & $0.00$ & $12188.28 \pm 142.94$ & $0.00$ & $9667.92 \pm 503.55$ & $0.00$ \\ \hline
PaS(G) \citep{Han23}        & $13.90 \pm 3.94$ & $0.05$ & $914.55 \pm 6.57$ & $0.05$ & $\textbf{12163.94} \pm 166.14$ & $\textbf{24.34}$ & $9680.52 \pm 498.68$ & $12.60$ \\
ConPaS(G) \citep{pmlr-v235-huang24f}      & $13.90 \pm 3.94$ & $0.05$ & $914.60 \pm 6.55$ & $0.00$ & $12149.87 \pm 137.89$ & $38.41$ & $9680.52 \pm 498.68$ & $12.60$ \\
L2O-DiffILO(G) \citep{geng2025differentiable} & $13.85 \pm 3.89$ & $0.00$ & $914.60 \pm 6.55$ & $0.00$ & $12136.92 \pm 176.17$ & $51.36$ & $9667.92 \pm 503.55$ & $0.00$ \\
Apollo-MILP(G) \citep{liu2025apollomilpalternatingpredictioncorrectionneural}  & $13.85 \pm 3.89$ & $0.00$ & $914.60 \pm 6.55$ & $0.00$ & $12121.56 \pm 140.25$ & $ 66.72$ & $9667.92 \pm 503.55$ & $0.00$ \\
\rowcolor{gray!20} \textbf{SRG(G)} \textbf{(Ours)} & $\textbf{13.85} \pm 3.89$ & $\textbf{0.00}$ & $\textbf{914.60} \pm 6.55$ & $\textbf{0.00}$ & $12125.72 \pm 165.77$ & $62.56$ & $\textbf{9667.92} \pm 503.55$ & $\textbf{0.00}$ \\ \hline
\end{tabular}}
\label{tab:large-scale}
\end{table*}


\section{Experiments}
\label{sec:expr}


\subsection{Settings}

\textbf{Dataset Benchmark.}
We evaluate on several public benchmarks covering same-scale, cross-scale, and cross-problem zero-shot tasks. \underline{\textbf{(1)}} Four widely used NP-hard public MILP benchmarks from \textbf{Ecole} \citep{prouvost2020ecole} with problem sizes in Table~\ref{tab:ecolesize}: 1) \textit{Set Covering (SC)}~\citep{chvatal1979greedy}; 2) \textit{Combinatorial Auction (CA)}~\citep{sandholm1999algorithm}; 3) \textit{Capacitated Facility Location (CFL)}~\citep{balinski1963integer}; and 4) \textit{Maximum Independent Set (MIS)}~\citep{Karp1972}. Each benchmark is provided at two scales, Medium and Large, based on the number of variables and constraints, for same-scale and cross-scale tasks; \underline{\textbf{(2)}} two benchmarks from \textbf{ML4CO} \citep{gasse2022machinelearningcombinatorialoptimization}: Item Placement (IP) and  Load Balancing (LB); \underline{\textbf{(3)}} very large public \textbf{MIPLIB}~\citep{miplib} benchmarks, each with more than $10{,}000+$ variables, for cross-problem zero-shot tasks.

%
\textbf{Baselines.}
We compare SRG with both exact MILP solvers and representative learning-based methods, comprising six baselines. \textit{Exact solvers:} \underline{\textbf{(1)}} \textit{SCIP}~\citep{bestuzheva2021scip}, a leading open-source MILP solver; \underline{\textbf{(2)}} \textit{Gurobi}~\citep{gurobi2023gurobi}, a state-of-the-art commercial MILP solver. For methods that rely on Gurobi as the downstream search engine, we use Gurobi~11.0 with 24 threads to ensure fair computational resource allocation. \textit{Learning-based baselines:} \underline{\textbf{(3)}} \textit{Predict-and-Search (PaS)}~\citep{Han23}, which predicts promising assignments and performs solver-based local search; \underline{\textbf{(4)}} \textit{Contrastive PaS (ConPaS)}~\citep{pmlr-v235-huang24f}, which improves PaS via contrastive learning. We also include two recent neural MILP solvers: \underline{\textbf{(5)}} \textit{Apollo-MILP}~\citep{liu2025apollomilpalternatingpredictioncorrectionneural}, which employs an alternating prediction-correction mechanism during search; and \underline{\textbf{(6)}} \textit{L2O-DiffILO}~\citep{geng2025differentiable}, a machine-learning-based MILP solver.


\textbf{Evaluation Metrics.}
We report two metrics in the main results: 
\underline{\textbf{(1)}} \textit{Obj}: The best primal bound objective value. \underline{\textbf{(2)}} $\mathrm{Gap}_{\mathrm{ref}}$ measures the primal-bound gap between each method and its baseline solver: $\mathrm{Obj}_{\mathrm{method}}-\mathrm{Obj}_{\mathrm{solver}}$ for minimization (e.g., SC, CFL) and $\mathrm{Obj}_{\mathrm{solver}}-\mathrm{Obj}_{\mathrm{method}}$ for maximization (e.g., MIS, CA), so lower is better and negative values indicate the method beats the baseline solver. We bold the best Obj and $\mathrm{Gap}_{\mathrm{ref}}$ among methods.

\textbf{Hyperparameters.} We report the number of variables and constraints in Tables~\ref{tab:ecolesize} to~\ref{tab:mpssize} in the Appx. For \textbf{SRG}, we list the number of Transformer blocks $L$, patch size $p$, trust region parameters $k_1,k_2,\Delta$ (as in prior work \citep{liu2025apollomilpalternatingpredictioncorrectionneural,Han23}), training steps $T_\text{Train}$, sampling steps $T_\text{sample}$, and guidance coefficients ($\gamma_c,\gamma_o$) and guidance strengths ($\rho_c,\rho_o$) for each benchmark scale in Table~\ref{tab:parameter} and Table~\ref{tab:sample-parameter}, both in the Appx.

\begin{wraptable}{r}{0.55\textwidth}
\vspace{0.1em}
\centering
\small
\setlength{\tabcolsep}{3pt}
\renewcommand{\arraystretch}{1.0}
\captionsetup{width=\linewidth,font=small}

\caption{\underline{\textbf{C}}ross \underline{\textbf{P}}roblem results with solving time limit of \underline{\textbf{6}}00s.}
\label{tab:gurobi_srg_obj}

\scalebox{0.7}{
\begin{tabular}{lcccc}
\toprule
\multirow{2}{*}{Problem}
& \multicolumn{2}{c}{Obj.}
& \multicolumn{2}{c}{Time(s)} \\
\cmidrule(lr){2-3} \cmidrule(lr){4-5}
& Gurobi & SRG (G) & Gurobi & SRG (G) \\
\midrule
IP($\downarrow$)
& $12.29 \pm 6.12$
& $\textbf{11.82} \pm 5.81$
& $600.16 \pm 0.05$
& $\textbf{600.30} \pm 0.07$ \\

LB($\downarrow$)
& $703.80 \pm 10.30$
& $\textbf{703.80} \pm 10.30$
& $600.59 \pm 0.18$
& $\textbf{589.36} \pm 23.27$ \\
\midrule
blp-ic98 ($\downarrow$)
& $4491.45$
& $\textbf{4491.45}$
& $19.0$
& $\textbf{17.7}$ \\

cmflsp50 \footnote{cmflsp50-24-8-8} ($\downarrow$)
& $55{,}789{,}390$
& $\textbf{55{,}789{,}390}$
& $\textbf{265.2}$
& $387.3$ \\

cbs-cta ($\downarrow$)
& $0.00$
& $\textbf{0.00}$
& $0.9$
& $\textbf{0.8}$ \\

tbfp-network ($\downarrow$)
& $24.16$
& $\textbf{24.16}$
& $135.2$
& $\textbf{123.0}$ \\
\bottomrule
\end{tabular}
}

\vspace{-1em}
\end{wraptable}

\subsection{Main Results}
\textbf{\underline{S}ame-\underline{S}cale Solving Effectiveness.}
We evaluate all methods on medium-scale benchmarks \citep{prouvost2020ecole}. As shown in Table~\ref{tab:medium-scale}, across different problem classes and downstream solvers, SRG achieves the best performance among learning-based methods in 7 out of 8 settings, while remaining competitive with SCIP and Gurobi. These results suggest that SRG can construct effective trust regions across diverse MILP problems.

\begin{wrapfigure}{r}{0.4\columnwidth}
\centering
\includegraphics[width=\linewidth]{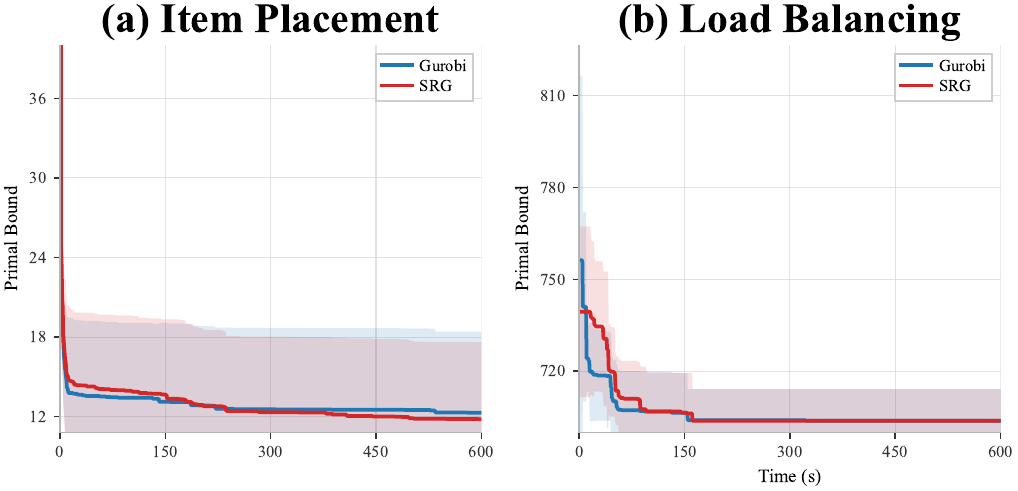}
\caption{\underline{\textbf{C}}ross \underline{\textbf{P}}roblem solving: Primal bound iteration trajectory for Item Placement and Load Balancing.}
\label{fig:crs}
\end{wrapfigure}

\textbf{\underline{\textbf{C}}ross-\underline{\textbf{S}}cale Generalization.}
We further evaluate the cross-scale transferability of SRG by directly applying the model pretrained on medium-scale benchmarks to large-scale test instances. 
As shown in Table~\ref{tab:large-scale}, compared with other learning-based baselines, SRG achieves the best or near-best objective values in most settings (7/8) under both SCIP and Gurobi. 
These results suggest that SRG captures transferable solution-generation patterns beyond the training scale and remains effective on larger MILP instances.

\textbf{\underline{\textbf{C}}ross-\underline{\textbf{P}}roblem Zero-Shot Results.}
We further evaluate the \textit{cross-scale} and \textit{cross-problem} generalization abilities of SRG on four unseen \textit{MIPLIB} instances~\citep{miplib} and two classical MILP benchmarks from \textit{ML4CO}~\citep{gasse2022machinelearningcombinatorialoptimization}, with sizes in Tables~\ref{tab:ml4cosize} and~\ref{tab:mpssize}.
We directly apply the SRG model pretrained on medium-scale MILP benchmarks without any fine-tuning. 
As shown in Table~\ref{tab:gurobi_srg_obj}, Fig.~\ref{fig:crs}, and Fig.~\ref{fig:crsmlplib} in the Appendix, SRG can transfer to unseen large-scale MILP instances and achieve competitive performance compared with the Gurobi baseline. 
In particular, SRG obtains a better objective on the Item Placement benchmark, while reducing solving time on Load Balancing and three MIPLIB instances. 
These results demonstrate the scalability and zero-shot transfer potential of SRG in cross-problem settings.



\subsection{Ablation Study}
\textbf{Relaxation Guidance.}
Removing the guidance terms in Eq.~(\ref{eq:modified_score_function_MAIN}) leads to consistent performance drops on Medium and Large CFL instances (Table~\ref{tab:ablation}), confirming that the relaxation-guided target provides essential feasibility and optimality signals for constructing high-quality search regions.

\textbf{Adaptation Rules for $\gamma_o,\gamma_c$.}
Without the adaptation rules, the $\epsilon$-space score-matching loss oscillates and rises to $\sim$140; with them, training stabilizes and converges to $\sim$1 (Fig.~\ref{fig:ablation_adapt}). This shows that the adaptation rules effectively normalize the guidance scale across timesteps and instances.

\subsection{Sensitivity Analysis}

\begin{wraptable}{l}{0.36\textwidth}
\vspace{-1.3em}
\centering
\caption{\underline{\textbf{A}}blation \underline{\textbf{S}}tudy on relaxation guidance on CFL (``M'' means Medium, and ``L'' means Large).}
\label{tab:ablation}
\resizebox{\linewidth}{!}{
\begin{tabular}{l|cc}
\hline
\multirow{1}{*}{Method} 
& \multicolumn{2}{c}{CFL(M, Gurobi)}  \\
& Obj($\downarrow$) & Time($\downarrow$) \\ \hline
\textbf{SRG} w/o Guide 
& $8878.85 \pm 441.06$  & $0.41 \pm 0.31$  \\
\textbf{SRG} \textbf{(Ours)} 
&  $\textbf{8858.28} \pm 446.19$ & $ 0.34 \pm 0.17$ \\ 
\hline
\multirow{1}{*}{Method}  
& \multicolumn{2}{c}{CFL (L, Gurobi)}\\ 
& Obj($\downarrow$) & Time($\downarrow$) \\ \hline
\textbf{SRG} w/o Guide  & $9688.42 \pm 502.69$ & $0.47 \pm 0.22$ \\
\textbf{SRG} \textbf{(Ours)}  & $ \textbf{9667.92} \pm 503.55 $ & $ 0.72 \pm 0.29 $  \\
\hline
\end{tabular}
}
\end{wraptable}
In this section, we conduct a sensitivity analysis on several key parameters in SRG with respect to both the training process (Fig.~\ref{fig:trainloss_ablationrho} in Appx.~\ref{app:ablation_add}) and the final solution quality (Table~\ref{tab:srg_guidance_ablation}).



\textbf{Impact of Guidance Strengths $\rho_o,\rho_c$.}
Table~\ref{tab:srg_guidance_ablation} shows that moderately increasing $\rho_c$ (from $0.3$ to $0.5$) improves performance, with further gains saturating at $1.0$, whereas increasing $\rho_o$ degrades performance, especially at $\rho_o=1.0$. This indicates that constraint guidance benefits from stronger weighting, while objective guidance must remain conservative to avoid dominating the diffusion signal.

\vspace{1em}
\begin{wraptable}{r}{0.5\textwidth}
\centering  
\caption{\underline{\textbf{S}}ensitivity \underline{\textbf{A}}nalysis results for different components in SRG on Medium CA. We report averaged metrics on \underline{\textbf{50}} instances. Except for the training-step experiment, all experiments use $T=50$.}
\label{tab:srg_guidance_ablation}

\resizebox{\linewidth}{!}{
\begin{tabular}{lcccc}
\toprule
Setting & $L, p$ & $\rho_o, \rho_c$ & $\gamma_o, \gamma_c$ & Obj.($\uparrow$) \\
\midrule
$\rho_c$ & $(12, 4)$ & $(0.3, 0.3)$ & $(2, 5)$ & $62729.85 \pm 1081.66$  \\
$\rho_c$ & $(12, 4)$ & $(0.3, 0.5)$ & $(2, 5)$ & $63054.66 \pm 1183.44$ \\
$\rho_c$ & $(12, 4)$ & $(0.3, 1.0)$ & $(2, 5)$ & $62854.61 \pm 1317.24$  \\ 
\midrule
$\rho_o$ & $(12, 4)$ & $(0.5, 0.3)$ & $(2, 5)$ & $62584.80 \pm 1051.49$  \\
$\rho_o$ & $(12, 4)$ & $(1.0, 0.3)$ & $(2, 5)$ & $62754.48 \pm 1161.98$  \\ 
\midrule
$\gamma_o,\gamma_c$ & $(12, 4)$ & $(0.3, 0.3)$ & $(2, 7)$ & $62770.34 \pm 1127.15$  \\
$\gamma_o,\gamma_c$ & $(12, 4)$ & $(0.3, 0.3)$ & $(5, 2)$ & $61657.34 \pm 1177.05$  \\
\midrule
$p$ & $(12, 6)$ & $(0.3, 0.3)$ & $(2, 5)$ & $62553.54 \pm 1194.54$  \\
$p$ & $(12, 8)$ & $(0.3, 0.3)$ & $(2, 5)$ & $62809.37 \pm 1173.52$  \\ 
\midrule
$L$ & $(8, 10)$ & $(0.3, 0.3)$ & $(2, 5)$ & $62869.37 \pm 1167.13$  \\
$L$ & $(10, 10)$ & $(0.3, 0.3)$ & $(2, 5)$ & $62627.60 \pm 1129.14$  \\ 
\midrule
$T_\text{Train}=30$ & $(12, 4)$ & $(0.3, 0.3)$ & $(2, 5)$ & $62875.54 \pm 1214.48$  \\
$T_\text{Train}=20$ & $(12, 4)$ & $(0.3, 0.3)$ & $(2, 5)$ & $62773.86 \pm 1189.16$   \\
\bottomrule
\end{tabular}
}
\vspace{-1em}
\end{wraptable}

\textbf{Impact of Guidance Weights $\gamma_o,\gamma_c$.}
SRG is stable under moderate changes in constraint-guidance weight: $(\gamma_o,\gamma_c)=(2,7)$ matches the default $(2,5)$. However, $(5,2)$ significantly degrades performance, confirming that feasibility guidance should dominate objective guidance.

\textbf{Impact of Training Steps $T_{\text{train}}$.}
Reducing $T_{\text{train}}$ from $50$ to $30$ or $20$ leads to a noticeable performance drop. This indicates that a sufficient number of training diffusion steps is important for learning an accurate guided score function. The similar results for $T_{\text{train}}=20$ and $30$ further suggests that too few training steps may place the model in a lower-performance regime.

\textbf{Impact of Transformer Blocks $L$ and Patch Size $p$.}
Varying the patch size $p$ or the number of transformer blocks $L$ only causes mild performance changes. The results under $p=6,8$ and $(L,p)=(8,10),(10,10)$ remain close to the default configuration, showing that SRG is not highly sensitive to moderate architectural variations.

Overall, SRG is robust to moderate architectural and guidance-parameter changes. Fig.~\ref{fig:trainloss_ablationrho} and Table~\ref{tab:srg_guidance_ablation} further suggest that balancing objective and constraint guidance, together with a sufficiently large $T_{\text{train}}$, is most important for final solution quality.


\begin{wrapfigure}{l}{0.35\textwidth}
\centering
\includegraphics[width=1\linewidth]{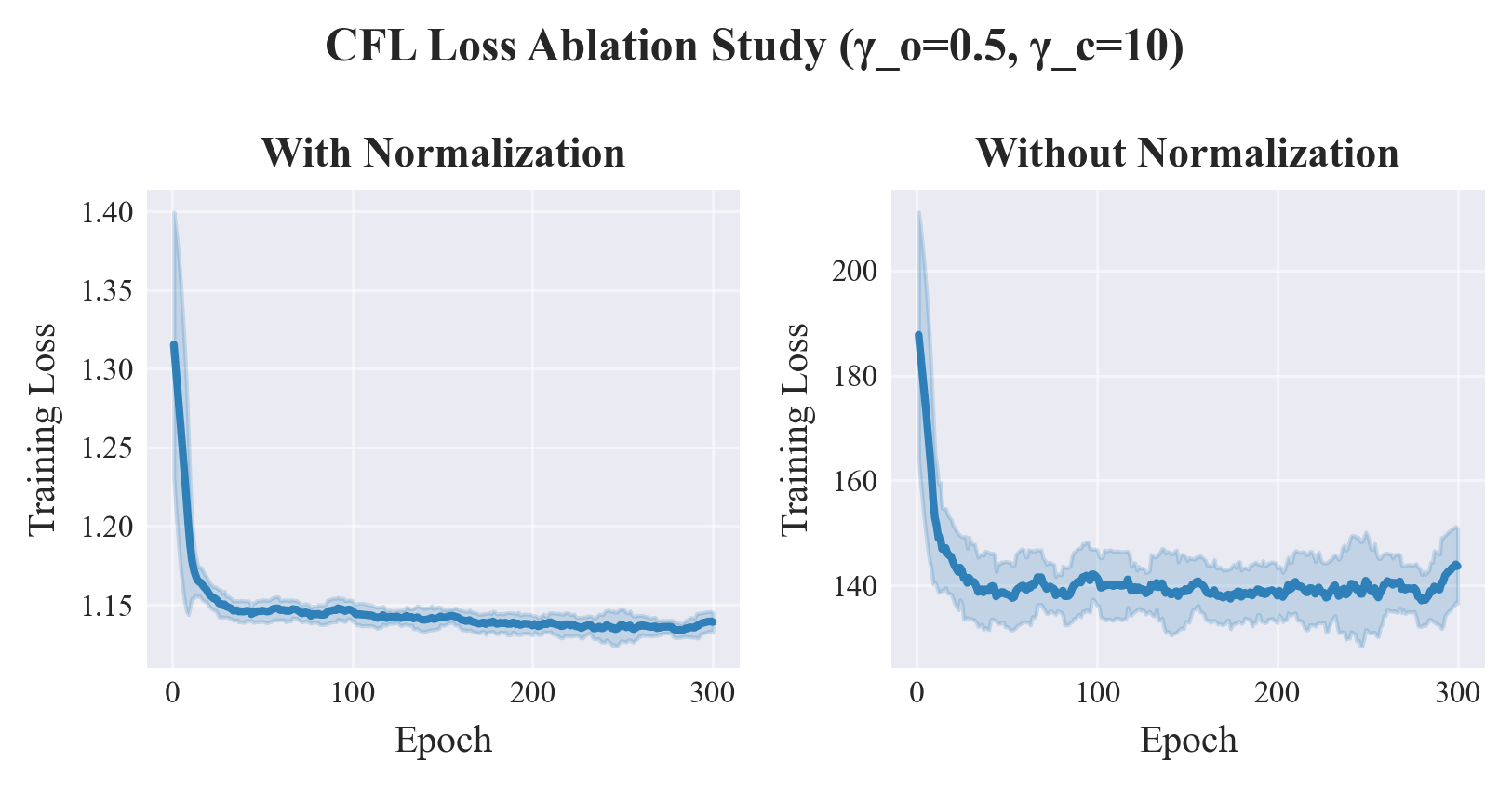}
\caption{\underline{\textbf{A}}blation \underline{\textbf{R}}esults for adaptation rules on $\gamma_o$ and $\gamma_c$.}
\label{fig:ablation_adapt}
\end{wrapfigure}

\subsection{Generation Analysis}
\textbf{Sample Diversity.} Despite each training instance providing only a single optimal label, our sampling strategy produces \textit{diverse} candidates (Figs.~\ref{fig:qica}--\ref{fig:qicfl} in Appendix~\ref{app:QUANT}), with diversity arising from the stochasticity of the reverse-time SDE and the directional shaping of the relaxation guidance. This expands the coverage of high-quality subproblems and improves robustness of the downstream search.

\textbf{Progressive Reduction.} Sampling trajectories on medium-scale CA and CFL (Figs.~\ref{fig:penaltyca}--\ref{fig:penaltycfl} in Appendix~\ref{app:QUANT}) show a progressive decrease in both objective and constraint-violation terms during generation, confirming that the learned model steers the solution distribution toward feasible, high-quality regions.

\textbf{Efficiency Analysis.} Under parallel sampling, our DDIM/DDPM sampler takes at most $0.53$ seconds on average across benchmarks (Table~\ref{tab:time_comparison}, Appendix~\ref{app:eff}), demonstrating the practicality of SRG for real-world deployment.

\section{Conclusion}
We present \textbf{S}RG, a \underline{\textbf{S}}core-based \underline{\textbf{R}}elaxation-guided \underline{\textbf{G}}enerative framework for MILP solving. By incorporating feasibility and optimality signals directly into score learning, SRG overcomes the single-point prediction limitation of existing predict-and-search methods. The learned model generates diverse, high-quality candidates by direct sampling, which are then used to construct compact trust-region subproblems for downstream solvers. Across multiple publicly available MILP benchmarks, SRG achieves competitive or improved solution quality compared with existing machine learning baselines and solver combinations. Furthermore, SRG demonstrates strong zero-shot transferability across cross-scale and cross-problem settings on unseen large-scale public benchmarks.

\paragraph{Limitations.} 
Due to computational constraints, the upper bound of SRG's capability is not fully explored in this work.


\bibliography{example_paper}
\bibliographystyle{unsrt}

\newpage
\appendix
\onecolumn

\section{Propositions}
\subsection{Proof of Proposition~\ref{prop:equivalence_MAIN}}
\label{app:lsde}

\begin{proof}
We prove the equivalence by showing that both optimization problems have the same objective function up to an additive constant. Define the target distribution:
\begin{equation}
\tilde{p}(x|\mathbf{g}) = \frac{1}{Z} p_{\text{data}}(x|\mathbf{g}) \exp(-\gamma_o \mathcal{O}(x) - \gamma_c \mathcal{P}(x))
\end{equation}
where $Z = \int p_{\text{data}}(x|\mathbf{g}) \exp(-\gamma_o \mathcal{O}(x) - \gamma_c \mathcal{P}(x)) dx$ is the normalization constant.

For the original optimization problem:
\begin{equation}
f_q = D_{\text{KL}}(q(x) \| p_{\text{data}}(x|\mathbf{g})) + \gamma_o \mathbb{E}_{x \sim q}[\mathcal{O}(x)] + \gamma_c \mathbb{E}_{x \sim q}[\mathcal{P}(x)]
\end{equation}

Expanding the KL divergence:
\begin{equation}
\begin{aligned}
f_q &= \mathbb{E}_{x \sim q}\left[\log \frac{q(x)}{p_{\text{data}}(x|\mathbf{g})}\right] + \gamma_o \mathbb{E}_{x \sim q}[\mathcal{O}(x)] + \gamma_c \mathbb{E}_{x \sim q}[\mathcal{P}(x)] \\
&= \mathbb{E}_{x \sim q}\left[\log q(x) - \log p_{\text{data}}(x|\mathbf{g}) + \gamma_o \mathcal{O}(x) + \gamma_c \mathcal{P}(x)\right] \\
&= \mathbb{E}_{x \sim q}\left[\log q(x) - \log p_{\text{data}}(x|\mathbf{g}) - \log \exp(-\gamma_o \mathcal{O}(x) - \gamma_c \mathcal{P}(x))\right] \\
&= \mathbb{E}_{x \sim q}\left[\log q(x) - \log(p_{\text{data}}(x|\mathbf{g}) \exp(-\gamma_o \mathcal{O}(x) - \gamma_c \mathcal{P}(x)))\right] \\
&= \mathbb{E}_{x \sim q}\left[\log q(x) - \log(Z \cdot \tilde{p}(x|\mathbf{g}))\right] \\
&= \mathbb{E}_{x \sim q}\left[\log \frac{q(x)}{\tilde{p}(x|\mathbf{g})}\right] - \log Z \\
&= D_{\text{KL}}(q(x) \| \tilde{p}(x|\mathbf{g})) - \log Z
\end{aligned}
\end{equation}
Therefore:
\begin{equation}
\boxed{\min_q f_q = \min_q D_{\text{KL}}(q(x) \| \tilde{p}(x|\mathbf{g})) - \log Z}
\end{equation}

Since $\log Z$ is constant with respect to $q$, the two optimization problems are equivalent.
\end{proof}

\section{Proof of Proposition~\ref{prop:reverse_sde}}
\label{app:proof}

\subsection{Setup and Assumptions}
\label{app:setup}

The refined target distribution is
\begin{equation}
    \tilde p_0(x \mid \mathbf{g}) := \frac{1}{Z}\, p_{\text{data}}(x \mid \mathbf{g})\, w(x), 
    \qquad Z := \int_{\mathbb{R}^n} p_{\text{data}}(x \mid \mathbf{g})\, w(x)\, dx.
\end{equation}

We make the following assumptions:
\begin{itemize}
    \item[\textbf{(A1)}] $p_{\text{data}}(\cdot \mid \mathbf{g}) > 0$ on $\mathbb{R}^n$ 
    with $\int p_{\text{data}}\, dx = 1$, and $\mathbb{E}_{x \sim p_{\text{data}}}[\|x\|^2] < \infty$.
    \item[\textbf{(A2)}] Since $\mathcal{O}, \mathcal{P} \ge 0$, we have 
    $w(x) \in (0, 1]$ for all $x \in \mathbb{R}^n$. Therefore $0 < Z \le 1$, 
    and $\tilde p_0$ is a well-defined probability density.
    \item[\textbf{(A3)}] The functions $\mathcal{O}$ and $\mathcal{P}$ are 
    Lipschitz continuous on $\mathbb{R}^n$ with constants 
    $L_\mathcal{O} = \|c\|_\infty$ and $L_\mathcal{P} = \lambda\|A\|_\infty$ 
    respectively. Hence $w$ is Lipschitz with constant $L_w$, and 
    $\nabla\mathcal{O}, \nabla\mathcal{P}$ exist almost everywhere; 
    at non-smooth points (L1 kinks and constraint boundaries), they are 
    interpreted as Clarke subgradients.
    \item[\textbf{(A4)}] The forward SDE has continuous noise schedule 
    $\beta : [0, T] \to \mathbb{R}_{>0}$ with 
    $\beta_{\min} \le \beta(t) \le \beta_{\max}$ for some constants 
    $0 < \beta_{\min} \le \beta_{\max} < \infty$.
\end{itemize}

The Gaussian transition kernel of the VP-SDE is
\begin{equation}
    q_t(x_t \mid x_0) = \mathcal{N}\big(x_t;\ \sqrt{\bar\alpha_t}\, x_0,\ 
    (1-\bar\alpha_t) I\big), \qquad 
    \bar\alpha_t = \exp\!\left(-\int_0^t \beta(s)\, ds\right) \in (0, 1].
\end{equation}
We write $\sigma_t^2 := 1 - \bar\alpha_t$ throughout. We denote by 
$\tilde p_t(\cdot \mid \mathbf{g})$ the marginal of $x_t$ starting 
from $x_0 \sim \tilde p_0$, and by $p_t(\cdot \mid \mathbf{g})$ the 
marginal starting from $x_0 \sim p_{\text{data}}$.

\subsection{Proof of Proposition~\ref{prop:reverse_sde}}
\label{app:proof_reverse_sde}

We prove the two claims separately: 
(i) the reverse-time SDE~\eqref{eq:reverse_sde1} preserves marginals; 
(ii) the score admits the exact decomposition~\eqref{eq:score_exact}.

\paragraph{Part (i): Reverse-time SDE.}

\begin{lemma}[Regularity of $\tilde p_t$]
\label{lem:regularity}
Under (A1)--(A4), for every $t \in (0, T]$:
(a) $\tilde p_t(x \mid \mathbf{g}) > 0$ on $\mathbb{R}^n$;
(b) $\tilde p_t(\cdot \mid \mathbf{g}) \in C^\infty(\mathbb{R}^n)$;
(c) $\nabla_x\log\tilde p_t(x \mid \mathbf{g})$ is well-defined and continuous.
\end{lemma}

\begin{proof}
Since $\tilde p_t = q_t * \tilde p_0$ is a Gaussian convolution with 
$q_t > 0$ and $\tilde p_0 > 0$ (by (A1)--(A2)), we have $\tilde p_t > 0$, 
proving (a). The Gaussian kernel is $C^\infty$ with all derivatives 
integrable, so by dominated convergence $\tilde p_t \in C^\infty$, 
proving (b). Combining (a) and (b), $\nabla_x\log\tilde p_t = \nabla_x\tilde p_t / \tilde p_t$ 
is well-defined and continuous, proving (c).
\end{proof}

By Anderson's theorem~\citep{anderson1982reverse}, any It\^o SDE 
$dx_t = f(x_t, t)\, dt + g(t)\, dW_t$ with smooth marginal 
$\tilde p_t > 0$ admits the time-reversed process
$dx_t = [f(x_t, t) - g(t)^2 \nabla_{x_t}\log \tilde p_t(x_t)]\,dt + g(t)\,d\bar W_t$, 
preserving marginals at every $t$. Specializing to the VP-SDE with 
$f(x, t) = -\tfrac{1}{2}\beta(t)\, x$ and $g(t) = \sqrt{\beta(t)}$ 
yields Eq.~\eqref{eq:reverse_sde1}; Lemma~\ref{lem:regularity} 
ensures Anderson's regularity conditions are met. \hfill$\square$

\paragraph{Part (ii): Score Decomposition.}

Starting from $x_0 \sim \tilde p_0 = \frac{1}{Z}\, p_{\text{data}}\, w$, 
the forward marginal at time $t$ is
\begin{equation}
    \tilde p_t(x_t \mid \mathbf{g}) 
    = \frac{1}{Z}\int q_t(x_t \mid x_0)\, p_{\text{data}}(x_0 \mid \mathbf{g})\, w(x_0)\, dx_0.
\label{eq:appA1}
\end{equation}
Define the unconditional marginal 
$p_t(x_t \mid \mathbf{g}) := \int q_t(x_t \mid x_0)\, p_{\text{data}}(x_0 \mid \mathbf{g})\, dx_0$ 
and the corresponding posterior 
$p(x_0 \mid x_t, \mathbf{g}) = q_t(x_t \mid x_0)\, p_{\text{data}}(x_0 \mid \mathbf{g}) / p_t(x_t \mid \mathbf{g})$ 
(well-defined since $p_t > 0$ by Lemma~\ref{lem:regularity}). 
Multiplying and dividing the integrand of~\eqref{eq:appA1} by 
$p_t(x_t \mid \mathbf{g})$ rewrites it as
\begin{equation}
    \tilde p_t(x_t \mid \mathbf{g}) 
    = \frac{1}{Z}\, p_t(x_t \mid \mathbf{g})\, 
    \mathbb{E}_{x_0 \sim p(x_0 \mid x_t, \mathbf{g})}[w(x_0)].
\label{eq:appA3}
\end{equation}
Taking the logarithm and differentiating with respect to $x_t$ 
(noting $\nabla_{x_t}\log Z = 0$ since $Z$ is independent of $x_t$):
\begin{equation}
    \nabla_{x_t}\log\tilde p_t(x_t \mid \mathbf{g}) 
    = \nabla_{x_t}\log p_t(x_t \mid \mathbf{g}) 
    + \nabla_{x_t}\log\mathbb{E}_{x_0 \sim p(x_0 \mid x_t, \mathbf{g})}[w(x_0)].
\label{eq:appA5}
\end{equation}
The interchange of differentiation and integration is justified by 
the Leibniz rule: the integrand is differentiable in $x_t$ with 
$w \le 1$ providing an integrable dominating bound. 
Equation~\eqref{eq:appA5} is precisely~\eqref{eq:score_exact}, 
completing the proof of Proposition~\ref{prop:reverse_sde}. \hfill$\square$

\section{Details for our 2D Toy Experiment}
\label{app:toyexperiments}
In this toy experiment, we use 1000 2D $[0,1]^2$ LP instances for training, with both training and sampling set to 50 steps. We use a convolutional generative network for learning the guided score. The mathematical formulation of the test instance is given in Eq.~\ref{eq:toy}:
\begin{equation}
\begin{aligned}
\min_{x_1, x_2} \quad & 1.1652 x_1 + 1.3948 x_2 \\
\text{subject to} \quad
& -0.9927 x_1 - 0.1208 x_2 \geq -0.8666 \\
& -0.9201 x_1 - 0.3918 x_2 \geq -1.0237 \\
& -0.5800 x_1 - 0.8146 x_2 \geq -1.0415 \\
& -0.0648 x_1 - 0.9979 x_2 \geq -0.8316 \\
&  0.5066 x_1 - 0.8622 x_2 \geq -0.3946 \\
&  0.9225 x_1 - 0.3860 x_2 \geq  0.0036 \\
&  0.9968 x_1 + 0.0797 x_2 \geq  0.1460 \\
&  0.8181 x_1 + 0.5750 x_2 \geq  0.4351 \\
&  0.5763 x_1 + 0.8172 x_2 \geq  0.3879 \\
& -0.0203 x_1 + 0.9998 x_2 \geq  0.1130 \\
& -0.5559 x_1 + 0.8313 x_2 \geq -0.1361 \\
& -0.8809 x_1 + 0.4733 x_2 \geq -0.4563 \\
& 0 \leq x_1 \leq 1 \\
& 0 \leq x_2 \leq 1
\label{eq:toy}
\end{aligned}
\end{equation}
Notably, we use a 4-layer ResNet as the denoising network with 2D conditional generation for simplified training. The generation results are obtained with $\gamma_o=0.1$, $\gamma_c=0.3$, and 50 sampling steps. Since the 2-D optimization problem is correspondingly small, the generative network employs the full score relaxation function, implemented with a 4-layer ResNet and simplified condition embedding. The mathematical formulation for our LP instance construction is in Figure~\ref{fig:obj}.
\begin{figure*}
\centering
\includegraphics[width=1\textwidth]{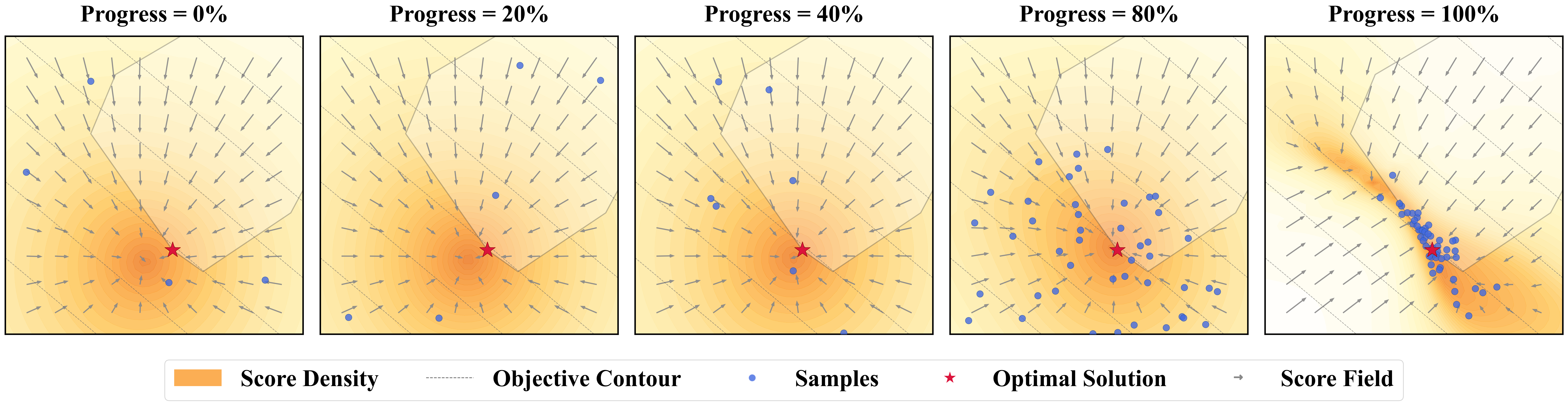}
\caption{\textbf{Toy experiments} with a relaxation-guided score}
\label{fig:lagtoy}
\end{figure*}

\begin{figure*}
\centering
\includegraphics[width=0.4\textwidth]{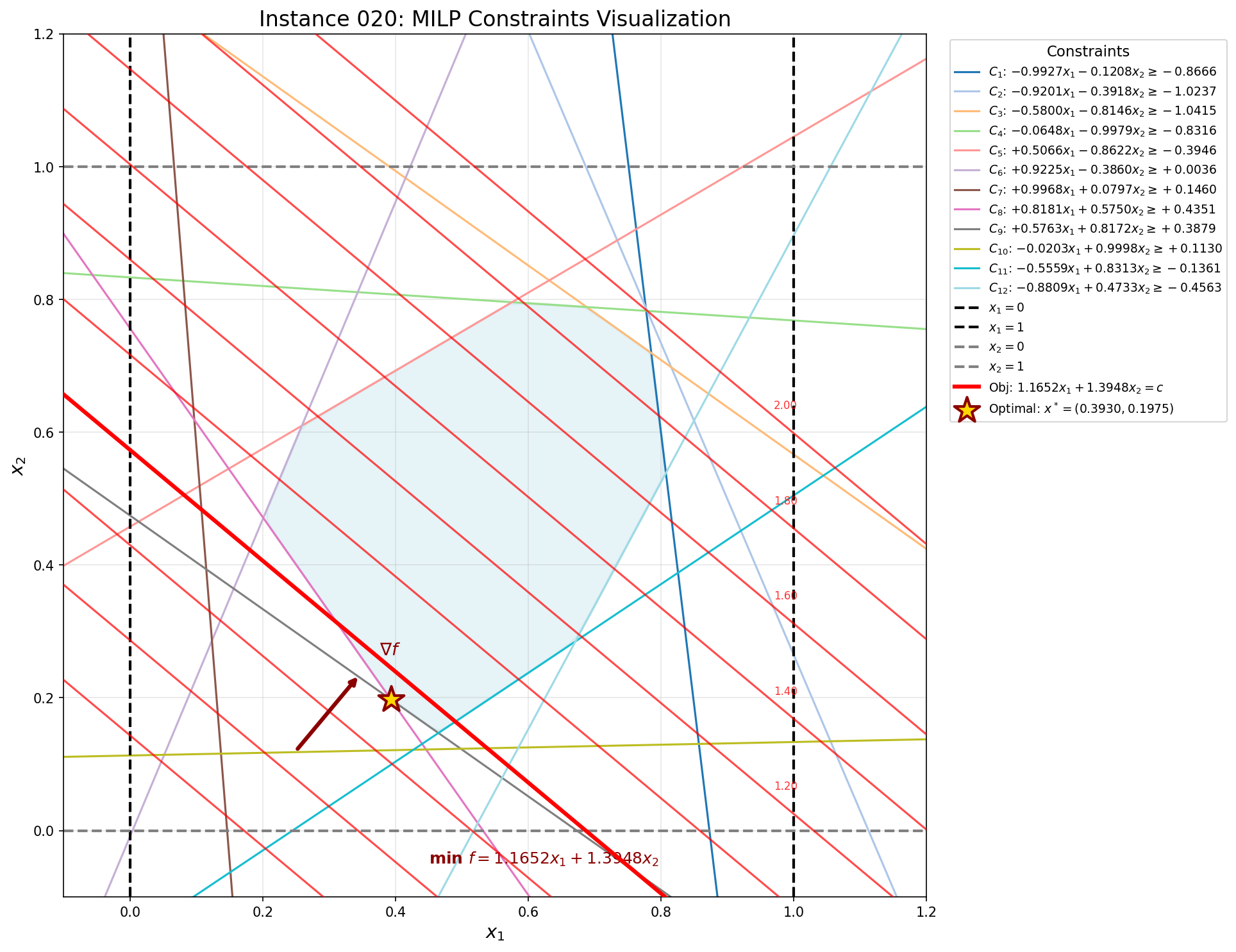}
\caption{\textbf{Toy 2D LP} instance with objective contours and constraints}
\label{fig:obj}
\end{figure*}

\section{MILP Benchmark Sizes}
\label{app:data}
In this section, we report the problem sizes for our MILP benchmarks, including four very large-scale \textbf{MIPLIB} benchmarks summarized in Table~\ref{tab:mpssize}, two classical MILP problems from \textbf{ML4CO}, and four problems at two scales from \textbf{ECOLE}.
\begin{table}[htbp]
\centering
\caption{Problem sizes for each benchmark (variables, constraints). All benchmark ground truth solutions are obtained by Gurobi within 600 seconds.}
\label{tab:ecolesize}
\vspace{6pt}
\scalebox{0.85}{
\begin{tabular}{lcccc cc}
\toprule
TYPE & \textbf{SC} & \textbf{MIS} & \textbf{CA} & \textbf{CFL} & Time Limit & Num. \\ \midrule
ECOLE (Middle) & (1000, 1000) & (1000, $\sim$ 4213) & (2000, $\sim$ 1375) & (1550, 1581)  & 100s & 50 \\
ECOLE (Large) & (2000, 1000) & (2000, $\sim$ 8478) & (2000, $\sim$ 2345) & (2550, 2601)  & 600s & 20 \\
\bottomrule
\end{tabular}}
\end{table}

\begin{table}[htbp]
\centering
\caption{Problem sizes for public benchmarks from ML4CO (variables, constraints).}
\label{tab:ml4cosize}
\vspace{6pt}
\scalebox{0.85}{
\begin{tabular}{lcccc cc}
\toprule
TYPE & \textbf{Item Placement} & \textbf{Load balancing} & Time Limit & Num. \\
\midrule
ML4CO & (1,083, 195) & (61,000, 64,363) & 600s & 5 \\
\bottomrule
\end{tabular}}
\end{table}

\begin{table}[htbp]
\centering
\caption{Problem sizes for public benchmarks from MIPLIB (variables, constraints).}
\label{tab:mpssize}
\vspace{6pt}
\scalebox{0.85}{
\begin{tabular}{lcccc}
\toprule
Size & \textbf{blp-ic98} & \textbf{cmflsp50-24-8-8} & \textbf{cbs-cta} & \textbf{tbfp-network} \\
\midrule
Variables & 13,640 & 16,392 & 24,793 & 72,747 \\ 
Discrete Variables & 13,550 & 1,392 & 2,467 & 72,747 \\
Continuous Variables & 90 & 15,000 & 22,326 & 0 \\ 
Constraints & 717 & 3,520 & 10,112 & 2,436 \\ 
\bottomrule
\end{tabular}}
\end{table}


\begin{wraptable}{r}{0.5\textwidth}
\vspace{-1em}
\centering
\scalebox{0.6}{
\begin{tabular}{l|cccc}
\hline
 & SC & IS & CA & CFL \\
\hline
Generate Time (M) 
& $0.46 \pm 0.03$ 
& $0.48 \pm 0.06$ 
& $0.45 \pm 0.03$ 
& $0.47 \pm 0.03$ \\
Generate Time (L) 
& $0.53 \pm 0.05$ 
& $0.48 \pm 0.04$ 
& $0.48 \pm 0.04$ 
& $0.47 \pm 0.05$ \\
\hline
\end{tabular}}
\caption{Generation time of SRG on medium and large-scale instances.}
\label{tab:time_comparison}
\vspace{-1em}
\end{wraptable}

\section{Implementation Details}
\label{app:impl}
All experiments are conducted on NVIDIA GPUs, including RTX A4000 (16GB), Tesla V100 (32GB), and RTX 3090 (24GB).

\subsection{Training Details and Results}
In this section, we report the training dynamics for Medium and Large benchmarks. The $\gamma_o,\gamma_c,T_\text{train},h,w,L,p,\rho_o,\rho_c$ used for training on four benchmarks are summarized in Table \ref{tab:parameter}. We use \underline{\textbf{5}}00 instances as the training set and train for \underline{\textbf{3}}00 epochs for all benchmarks. We visualize the training loss in Figure~\ref{fig:training_loss}.

\begin{figure*}
\centering
\includegraphics[width=1\textwidth]{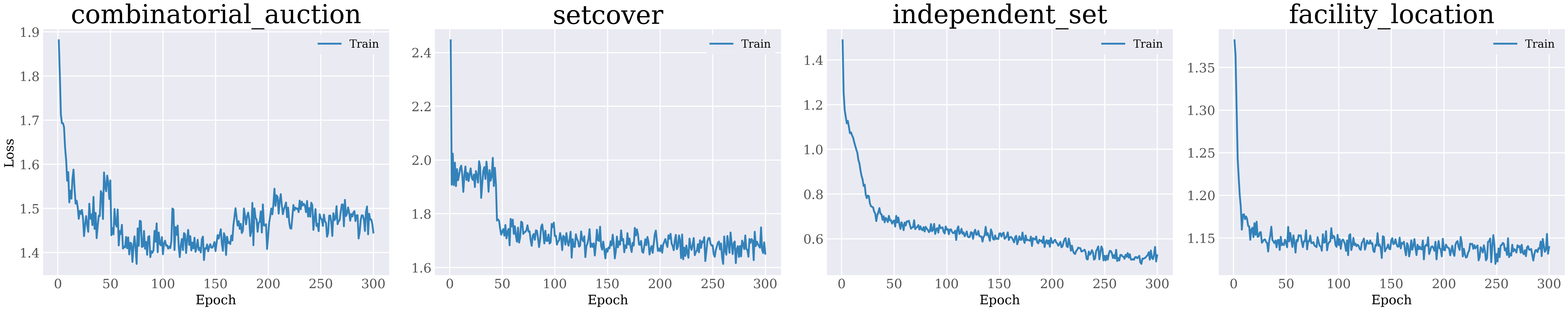}
\caption{Training curves on four benchmarks.}
\label{fig:training_loss}
\end{figure*}

\begin{table}[htbp]
\centering
\caption{Training parameters for main results.}
\label{tab:parameter}
\vspace{6pt}
\scalebox{0.9}{
\begin{tabular}{lcccc}
\toprule
Medium & SC & CA & IS & CFL \\
\midrule
$(h,w)$ & (25,40) & (40,50) & (25,40) & (31,50)\\
$T_\text{Train}$ & 50 & 50 & 20 & 20   \\
$L$ & 12 & 12 & 12 & 12   \\
$p$ & 10 & 4 & 4 & 4   \\
$(\gamma_o, \gamma_c)$ & (5, 10)  & (2, 5)  & (3,5)  & (0.05,10) \\
$(\rho_o, \rho_c)$ & (0.3,0.3) &(0.3,0.5) & (0.3,0.3)& (0.3,0.3)\\
\bottomrule
\end{tabular}}
\end{table}

\subsection{Sampling Details}
In this section, we report the sampling parameters in Table \ref{tab:sample-parameter} for ($k_1,k_2,\Delta,T_\text{sample}$) in our score-based model and the search phase (following \citep{Han23}). 

\begin{table}[htbp]                                                                  
\centering                                                                            \caption{Sampling parameters for main results, split by solver.}                         
\label{tab:sample-parameter}                                                             
\vspace{6pt}                                                                             
\resizebox{\textwidth}{!}{%
\begin{tabular}{l|cccc|cccc|cccc|cccc}                                                   
\toprule                                                                                 
& \multicolumn{8}{c|}{Medium} & \multicolumn{8}{c}{Large} \\                             
\cmidrule(lr){2-9}\cmidrule(lr){10-17}                                                   
& \multicolumn{4}{c|}{Gurobi} & \multicolumn{4}{c|}{SCIP}                                
& \multicolumn{4}{c|}{Gurobi} & \multicolumn{4}{c}{SCIP} \\                              
& SC & CA & IS & CFL & SC & CA & IS & CFL                                                
& SC & CA & IS & CFL & SC & CA & IS & CFL \\                                             
\midrule                                                                                 
$T_{\text{sample}}$ & 20  & 20  & 20  & 20  & 20  & 20  & 20  & 20  & 20 & 20  & 20  & 20
& 20 & 20  & 20  & 20  \\                                                              
$k_1$               & 20  & 150 & 500 & 100 & 20  & 150 & 200 & 100 & 50 & 50  & 150 &   
150 & 50 & 150 & 150 & 150 \\                                                            
$k_2$               & 400 & 800 & 50  & 200 & 400 & 800 & 400 & 200 & 300& 200 & 800 &   
300 & 300& 800 & 800 & 300 \\                                                            
$\Delta$            & 300 & 500 & 400 & 150 & 300 & 500 & 400 & 150 & 200& 200 & 500 &   
200 & 200& 500 & 500 & 200 \\                                                            
$N_{\text{sample}}$ & 8   & 8   & 8   & 8   & 8   & 8   & 8   & 8   & 8  & 8   & 8   & 8 
& 8  & 8   & 8   & 8   \\                                                              
\bottomrule                                                                              
\end{tabular}}                                                                           
\end{table}

\subsection{\underline{E}fficiency \underline{A}nalysis}
\label{app:eff}
In this section, we report generation time for our score-based models. We report the runtime across four benchmarks at two different scales, as shown in Table~\ref{tab:time_comparison}.

\section{Additional Results}
\label{appx:add}
This section supplements the experimental results in the main text. We report \underline{\textbf{P}}rimal \underline{\textbf{Bound}} iterations for cross-problem solving tasks (in Appx.~\ref{app:mipresults}), \underline{D}etailed \underline{P}erformance \underline{I}mpact (in Appx.~\ref{app:ablation_add}), \underline{Q}uantitative \underline{G}eneration \underline{P}erformance (in Appx.~\ref{app:QUANT}).

\subsection{\underline{\textbf{P}}rimal \underline{\textbf{Bound}} iterations for Cross Problem Solving Tasks}
\label{app:mipresults}
In this section, we show the detailed primal bound iterations for MIPLIB problems in Figure~\ref{fig:crs} and Figure~\ref{fig:crsmlplib}.

\begin{figure*}
\centering
\includegraphics[width=1\textwidth]{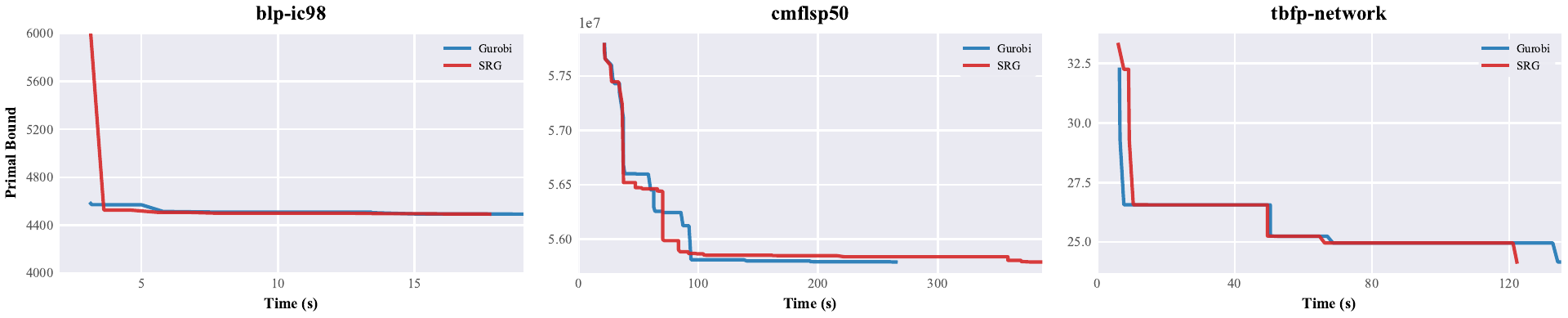}
\caption{Primal bound iterations on MIPLIB instances.}
\label{fig:crsmlplib}
\end{figure*}

\subsection{\underline{Q}uantitative \underline{G}eneration \underline{P}erformance}
\label{app:QUANT}
In this section, we visualize the sampling/generation process on medium-scale combinatorial auction (CA) and CFL problems. Specifically, we visualize the objective value (after projection onto the mixed-integer constraints), the average constraint violation, and the sparsity of the solution along the sampling trajectory. The quantitative metrics along the sampling process are reported in Fig. \ref{fig:penaltycfl} and Fig. \ref{fig:penaltyca}; the generation visualizations are shown in Fig. \ref{fig:qica} and Fig. \ref{fig:qicfl}.

\begin{figure*}
\centering
\includegraphics[width=1\textwidth]{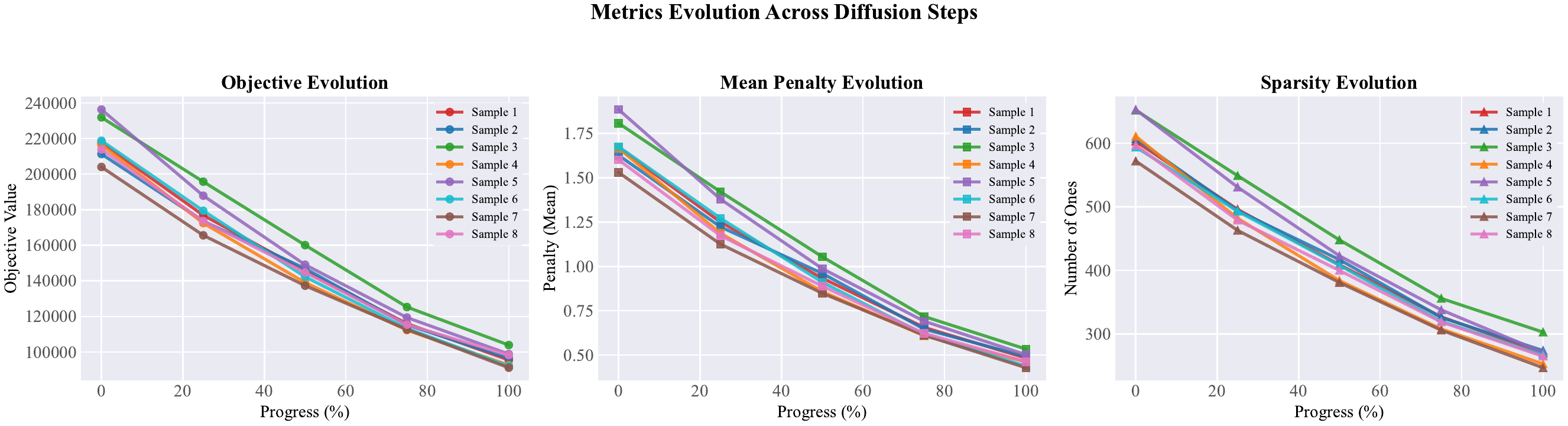}
\caption{ Obj/Penalty/Sparsity metric trajectories during the iterative generation process for a medium-scale Combinatorial Auction instance.}
\label{fig:penaltyca}
\end{figure*}

\begin{figure*}
\centering
\includegraphics[width=1\textwidth]{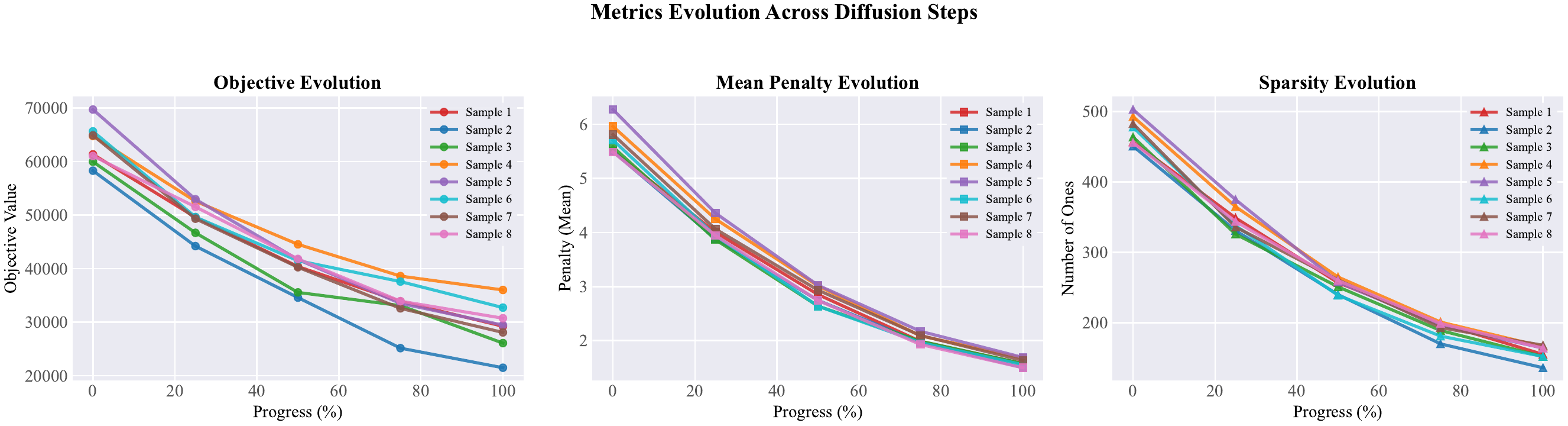}
\caption{ Obj/Penalty/Sparsity metric trajectories during the iterative generation process for a medium-scale Facility Location instance.}
\label{fig:penaltycfl}
\end{figure*}

\begin{figure*}
\centering
\includegraphics[width=1\textwidth]{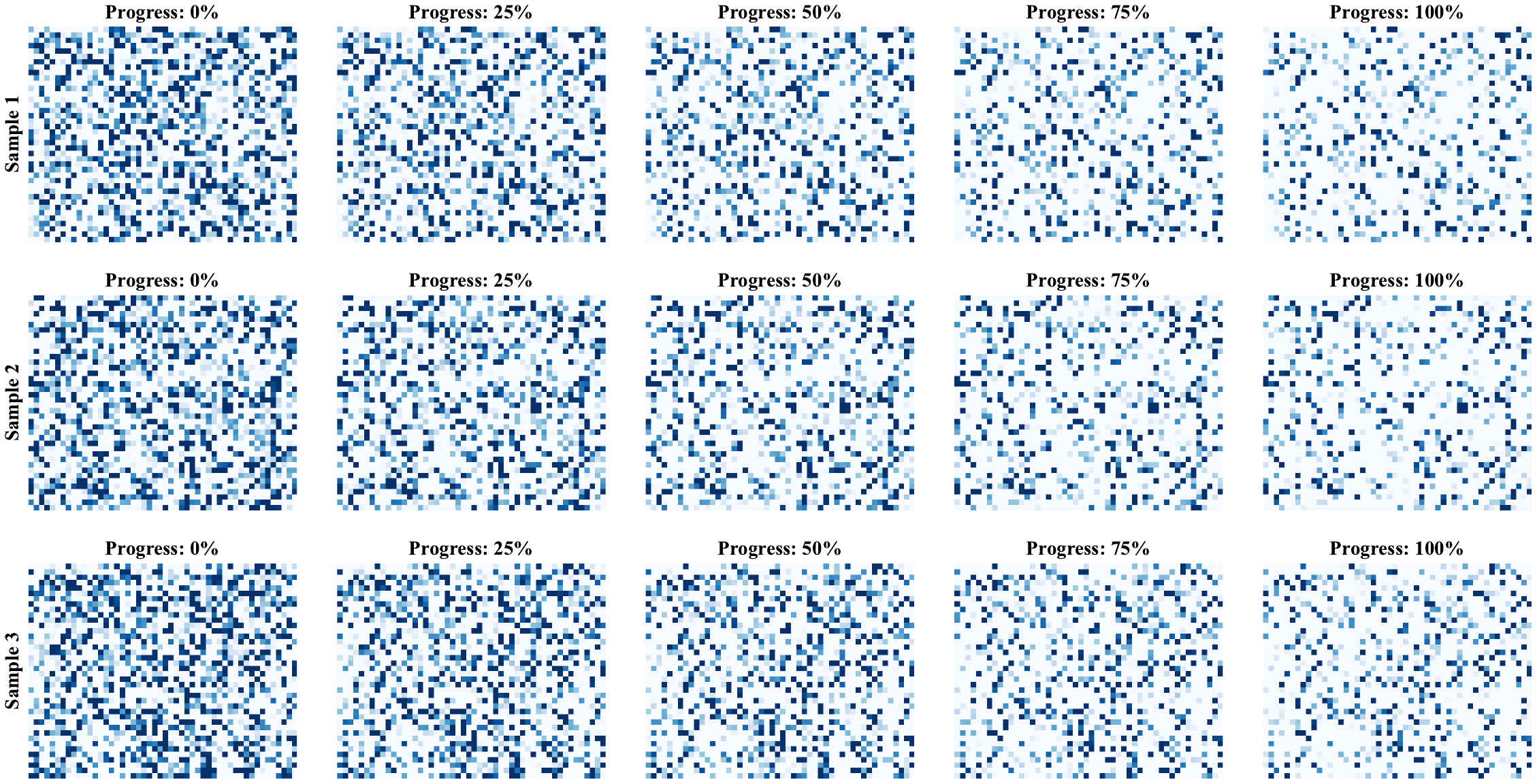}
\caption{Qualitative illustration of the generation process for medium-scale Combinatorial Auction.}
\label{fig:qica}
\end{figure*}

\begin{figure*}
\centering
\includegraphics[width=1\textwidth]{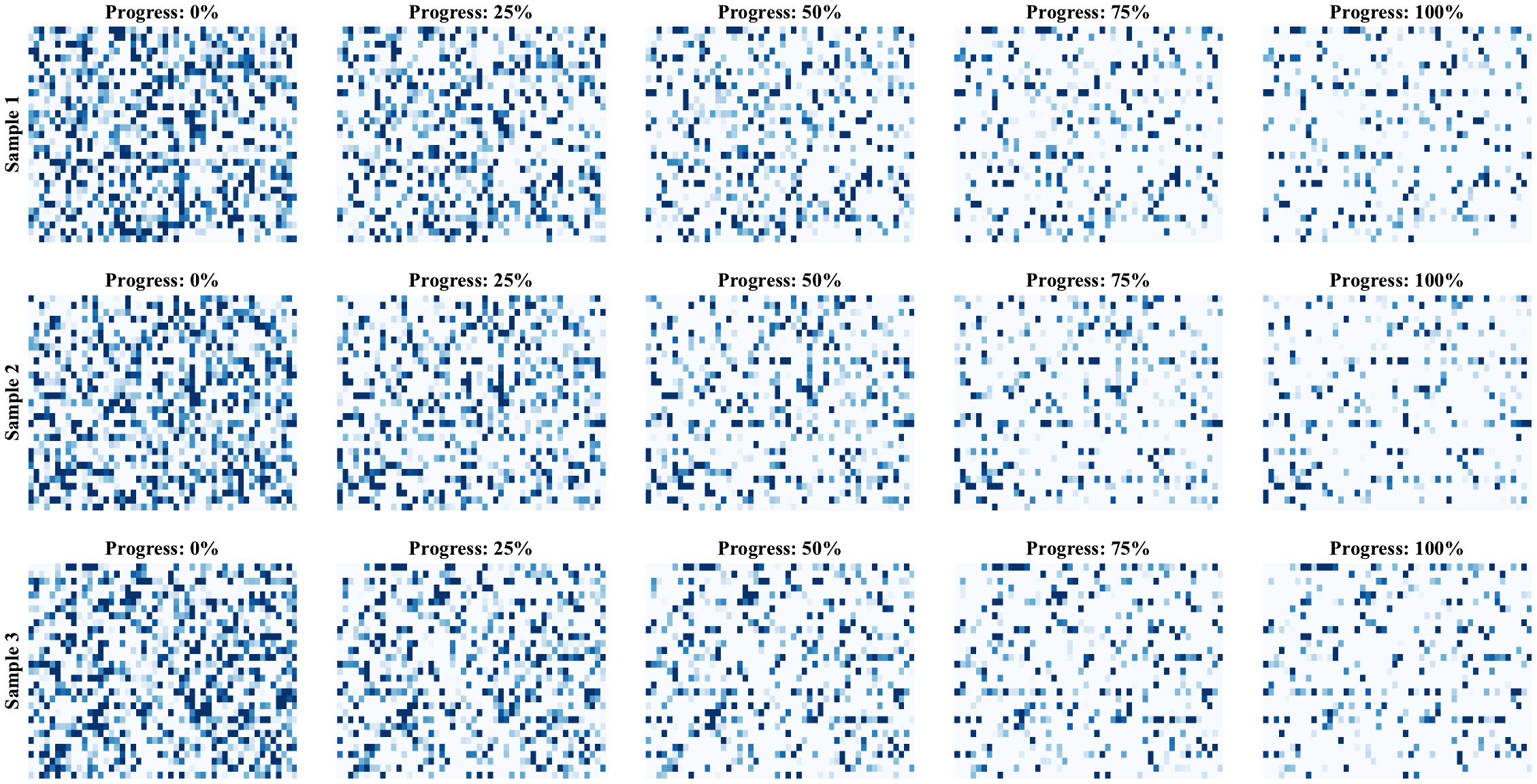}
\caption{Qualitative illustration of the generation process for medium-scale Facility Location.}
\label{fig:qicfl}
\end{figure*}

\subsection{\underline{D}etailed \underline{P}erformance \underline{I}mpact}
\label{app:ablation_add}
This section visualizes how the \underline{L}oss \underline{D}ynamics vary across different values of each parameter.

\textbf{\underline{\textbf{T}}raining Dynamics Impact} 
We conduct a sensitivity analysis on the key hyperparameters of our 
diffusion model, including the number of training timesteps $T$, the 
guidance scales $\gamma_o$ and $\gamma_c$, the number of Transformer 
layers $L$, the patch size $p$, and the loss-weighting coefficients 
$\rho_o$ and $\rho_c$. The results are reported in Fig.~\ref{fig:trainloss_ablationrho}. Both the convergence speed and the final converged value of the loss function are 
sensitive to these hyperparameters, indicating that 
they should be carefully tuned to ensure stable training.

\section{Broader Impacts.}
\label{app:impact}
This work aims to improve the efficiency of solving mixed-integer linear programming problems, which may benefit applications such as logistics, scheduling, resource allocation, and energy system optimization. By generating high-quality candidate solutions, the proposed method can potentially reduce computational cost and improve solver performance in large-scale optimization tasks. At the same time, MILP solvers are often used in high-stakes decision-making pipelines, where incorrect or biased problem formulations may lead to undesirable outcomes. Therefore, SRG should be used as a decision-support tool together with standard solver validation, domain-specific constraints, and appropriate human oversight, rather than as an unchecked automated decision-maker.


\begin{figure*}
\centering
\includegraphics[width=1\textwidth]{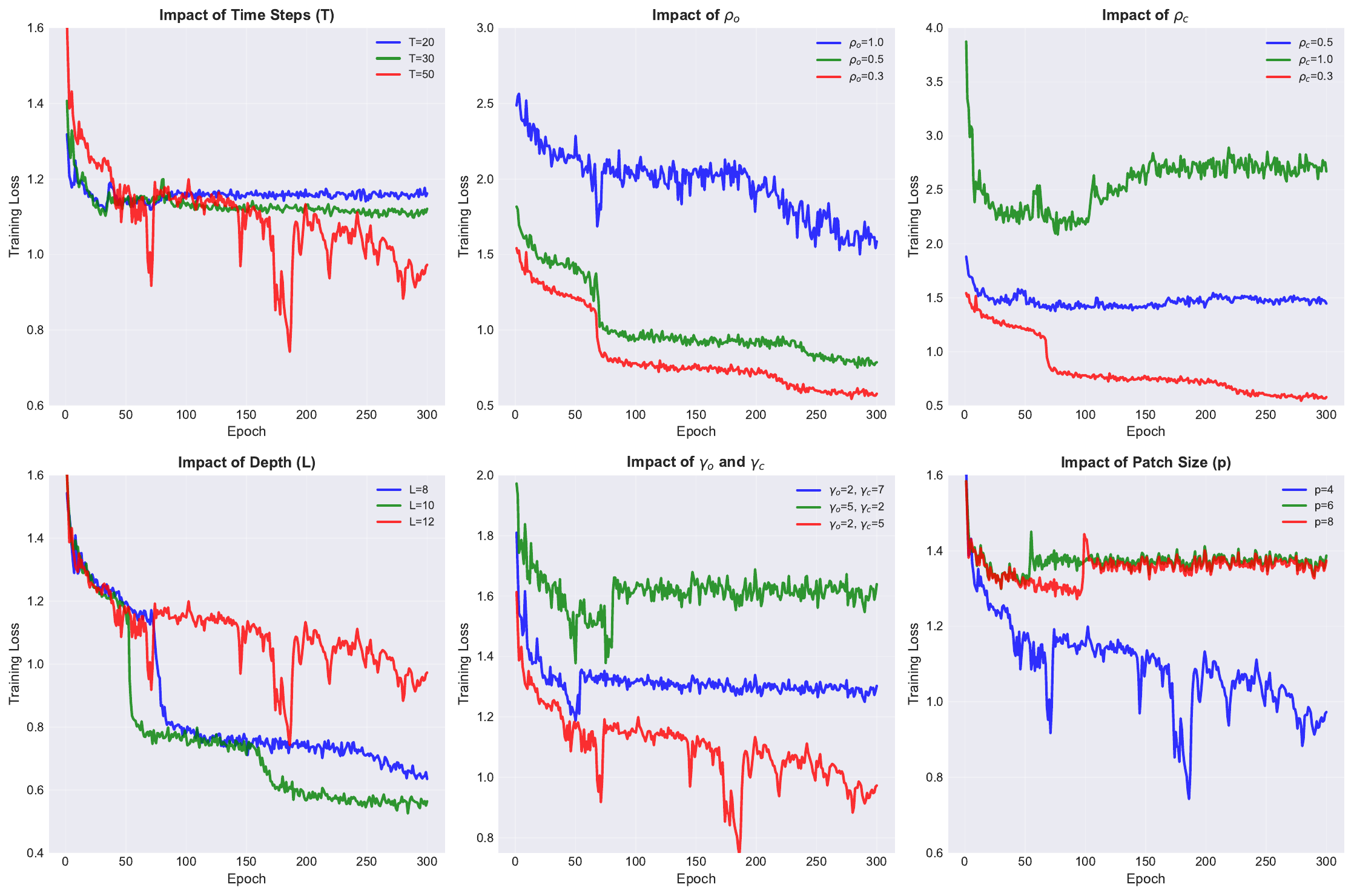}
\caption{Influence of different parameters on the target-score approximation loss on the medium-scale CA benchmark.}
\label{fig:trainloss_ablationrho}
\end{figure*}


\end{document}